\documentclass[12pt]{amsart}
\usepackage{times}
\DeclareMathAlphabet{\mathpzc}{OT1}{pzc}{m}{it}

\usepackage[T1]{fontenc}
\usepackage{dsfont}
\usepackage{mathrsfs}
\usepackage[colorlinks]{hyperref}
\usepackage{xcolor}
\usepackage[a4paper,asymmetric]{geometry}
\usepackage{mathscinet}
\usepackage{fullpage}
\usepackage{latexsym}
\usepackage{graphicx}
\usepackage{epstopdf}
\usepackage{amsthm}
\usepackage{amssymb}
\usepackage{amsfonts}
\usepackage{amsmath}
\usepackage{cite}
\usepackage{float}  
\usepackage{subfigure} 
\usepackage{caption} 
\usepackage[ruled,linesnumbered]{algorithm2e}
\usepackage{natbib}
\setcitestyle{authoryear,round}

\newtheorem{theorem}{Theorem}[section]
\newtheorem{thm}[theorem]{Theorem}

\newtheorem{lemma}[theorem]{Lemma}
\newtheorem{lem}[theorem]{Lemma}
\newtheorem{proposition}[theorem]{Proposition}

\newtheorem{corollary}[theorem]{Corollary}
\newtheorem{assumption}[theorem]{Assumption}
\theoremstyle{definition}

\newtheorem{defn}[theorem]{Definition}

\theoremstyle{remark}
\newtheorem{remark}[theorem]{Remark}
\newtheorem{rem}[theorem]{Remark}
\numberwithin{equation}{section}

\DeclareMathAlphabet{\mathpzc}{OT1}{pzc}{m}{it}

\newcommand{\dif}{\mathrm{d}}

\newcommand{\E}{\mathbb{E}}            

\newcommand{\e}{\varepsilon}

\newcommand{\N}{\mathbb{N}}
\newcommand{\R}{\mathbb{R}}
\newcommand{\Z}{\mathbb{Z}}
\newcommand{\PP}{\mathbb{P}}

\newcommand{\Be}{\begin{equation}}
	\newcommand{\Ee}{\end{equation}}
\newcommand{\Bes}{\begin{equation*}}
	\newcommand{\Ees}{\end{equation*}}
\newcommand{\Bey}{\begin{eqnarray}}
	\newcommand{\Eey}{\end{eqnarray}}
\newcommand{\Beys}{\begin{eqnarray*}}
	\newcommand{\Eeys}{\end{eqnarray*}}
\newcommand{\BT}{\begin{thm}}
	\newcommand{\ET}{\end{thm}}
\newcommand{\Bp}{\begin{proof}}
	\newcommand{\Ep}{\end{proof}}
\newcommand{\BL}{\begin{lem}}
	\newcommand{\EL}{\end{lem}}
\newcommand{\BP}{\begin{proposition}}
	\newcommand{\EP}{\end{proposition}}
\newcommand{\BC}{\begin{corollary}}
	\newcommand{\EC}{\end{corollary}}
\newcommand{\BR}{\begin{rem}}
	\newcommand{\ER}{\end{rem}}
\newcommand{\BD}{\begin{defn}}
	\newcommand{\ED}{\end{defn}}
\newcommand{\BI}{\begin{itemize}}
	\newcommand{\EI}{\end{itemize}}

\usepackage{marginnote}
\begin{document}

	\title[ ]
{Distribution estimation and change-point estimation for time series via DNN-based GANs}

\author[J. Lu]{Jianya Lu}
\address[J. Lu]{Department of Mathematical Sciences, University of Essex}
\email{jianya.lu@essex.ac.uk}

\author[Y. Mo]{Yingjun Mo}
\address[Y.~Mo]{Department of Mathematics, Faculty of Science and Technology, University of Macau, Macau S.A.R., China }
\email{yc27477@um.edu.mo}

\author[Z. Xiao]{Zhijie Xiao}
\address[Z. Xiao]{Department of Economics, Boston College}
\email{zhijie.xiao@bc.edu}

\author[L.~Xu]{Lihu Xu}
\address[L.~Xu]{Department of Mathematics, Faculty of Science and Technology, University of Macau, Macau S.A.R., China }
\email{lihuxu@um.edu.mo}

\author[Q.  Yao]{Qiuran Yao}
\address[Q.~Yao]{Department of Mathematics, Faculty of Science and Technology, University of Macau, Macau S.A.R., China }
\email{yb97478@um.edu.mo}
\maketitle
\begin{abstract}
The generative adversarial networks (GANs) have recently been applied to estimating the distribution of independent and identically distributed data, and have attracted a lot of research attention. In this paper, we use the blocking technique to demonstrate the effectiveness of GANs for estimating the distribution of stationary time series. Theoretically, we derive a non-asymptotic error bound for the Deep Neural Network (DNN)-based GANs estimator for the stationary distribution of the time series. Based on our theoretical analysis, we propose an algorithm for estimating the change point in time series distribution. The two main results are verified by two Monte Carlo experiments respectively, one is to estimate the joint stationary distribution of $5$-tuple samples of a 20 dimensional AR(3) model, the other is about estimating the change point at the combination of two different stationary time series. A real world empirical application to the human activity recognition dataset highlights the potential of the proposed methods.\\ \par\
\end{abstract}
\section{Introduction}
Let $\{\boldsymbol X_n\}_{n\in\N}$ be a stationary time series with the stationary measure $\pi$ {  satisfying appropriate regularity assumptions}, see the details in Section \ref{sec: Assumptions}. We aim in this paper to investigate the validity of DNN-based generative adversarial networks (GANs) for estimating $\pi^l$, the joint stationary distribution of $l$-tuple samples, of this time series, and propose an algorithm based on our theoretical result to estimate the change-point in time series.

Estimation of distribution plays an important role in data analysis. Many traditional methods {     on distributional estimation} are based on nonparametric kernel  {methods}, and suffer from the curse of dimensionality. In recent years, machine learning algorithms such as the GANs emerged as important techniques in distribution estimation and {     have demonstrated} excellent performance in high dimensional problems. 

GANs are a class of deep learning methods designed  to estimate the distribution of data [\citet{goodfellow2014generative}]. The approach is usually implemented by training two neural networks with conflicting objectives, one generator ($\boldsymbol {g}$), and one discriminator (d), forcing each other to improve. The generator is designed to produce {     an estimator that approximates the sample distribution}, while the discriminator  measures the distance between generated samples and the real data.

In the research of GANs, extensive theoretical and empirical analysis have been established for independent and identically distributed  (i.i.d.) data, however limited
work has exploited them for estimating the joint distribution of time-series data, which {  is important in many applications}. In addition, the evaluation of GANs remains a largely-unsolved problem, researchers often rely on visual evaluation of generated examples, an approach that is both impractical and inappropriate for multi-dimensional time series. 
There have been {     a few} results related to the GANs for time series, but most of them only showed the successes in experiments without a rigorous justification, see more details in the literature review below. 

Our first main result contributes to the theory of the GAN method  for estimating the joint distribution of stationary time-dependent time series. We establish the convergence rate of the proposed estimator measured by integral probability metric under the H\"older evaluation function class. {     We consider the joint distribution of a "block" or a "group" of samples and develop approximation error bounds for the DNN-based GANs. This result is very important.  In practice, a wide range of time series procedures are constructed based on  "block"s or "group"s  of time series observations. For example, blockwise bootstrap methods use blocks of time series observations to capture the dependence structure in the original data. Our result has great potential in statistical applications in time series analysis.} Our Theorem \ref{thm:main} below provides a theoretical   justification for the application of GANs in time series under appropriate regularity conditions.    

{     An important issue in statistical applications is the potential structural change in the underlying stationary distribution in a time series. For example, } there might have been abrupt variations in time series data. {     For this reason,} the estimation of  change-points gains popularity in modeling and prediction of time series, for example, in finance, biology, engineering, etc. Most of the traditional change-point estimation methods conduct the hypothesis test for the density function to monitor the change point [\citet{harchaoui2009regularized}]. However, these methods may have low estimation efficiency when encountering high-dimensional time series data. Some change-point estimation methods for high-dimensional {     multivariate} time series indirectly use the model selection or testing {     based on} the parametric models [\citet{cho2015multiple}, \citet{chen2015graph}]. The algorithms of parametric models are hard to extend to the massive data case due to the computational cost and {     in particular,} some strong assumptions [\citet{chen2015graph}]. Motivated by our first main result which confirms that GANs can directly learn the distribution of high-dimensional complex time series data, we propose a simple algorithm without any {     parametric} constraint to estimate the change point directly by tracking the variations in the distribution of the time series based on GAN.

\subsection{Literature review}

The research of the GANs architecture for sequential data has received more and more attention. In recent years, many methods of modeling time series with GANs have been proposed, such as continuous recurrent neural networks with adversarial training (C-RNN-GAN) [\citet{mogren2016c}], Recurrent Conditional GAN (RCGAN) [\citet{esteban2017real}], time series-GAN [\citet{yoon2019time}], Conditional Sig-Wasserstein GANs (SigCWGAN) [\citet{ni2020conditional}]. However, they only showed the success of the GAN-based methods for realistic sequence generation in experiments but there is no theory to support the validity of this model. We refer to the survey [\citet{brophy2021generative}] for more details about GANs in time series. 

Compared with the vast number of theoretical results for neural networks used to approximate functions, the use of neural networks to express distributions is much less theoretically understood. Some literature focus on how well GAN and its variants can express probability distributions. \citet{lu2020universal} studied the universal approximation property of deep neural networks for representing probability distributions, \citet{zhu2020deconstructing} studied a fundamental trade-off in the approximation error and statistical error in GANs, \citet{huang2022error} studied the convergence rates of GAN estimators under the integral probability metric with the H\"older function class.  \citet{liang2021well} established the optimal minimax rates for distribution estimation under GANs in view of  nonparametric density estimation. \citet{zhang2017discrimination} explored the  generalization capacity of GANs and extended the generalization bounds for neural net distance. 

For the theoretical analysis of GANs, many  scholars consider the probability distribution estimating of i.i.d. data by GANs. \citet{gao2022approximation} estimated the error bound of the approximation of  Wasserstein GANs using GroupSort neural networks as discriminators.  \citet{bai2018approximability} used GANs with some special discriminators to learn distributions  of i.i.d. samples in Wasserstein distance and KL-divergence. \citet{liu2021non} introduced a nearly sharp bound for the bidirectional GAN estimation error under the Dudley distance. We refer the reader to [\citet{taghvaei20192,chae2021likelihood}] for more details. Although a lot of work on the theoretical basis of GAN has appeared in academia, there seems very few of theoretical analysis for generating time-dependent data. For the estimation of joint distribution using GANs. \cite{srinivasan2022time} estimate the joint distribution and use it to generate synthetic time series data. \cite{pu2018jointgan} learn the joint distribution of multiple random variables from the data that samples from conditional distribution and marginal distribution. \citet{yoon2019time} estimate the joint distribution via a learned embedding space with both supervised losses. However, there seems few theoretical result for joint distribution estimation using GANs.

{ The traditional high-dimensional change-point estimation is around two aspects: model selection and testing. \citet{tibshirani2008spatial}, \citet{li2016panel}, \citet{harchaoui2010multiple} utilized the Lasso-type regularization model to estimate the change points and improve the computational efficiency. For the aspect of testing, \citet{harchaoui2009regularized} used a regularized kernel-based test statistic to test the consistency of  distributions between time-series samples. \citet{enikeeva2019high} considered the hypothesis
	test to test the change point in a set of high-dimensional Gaussian vectors. \citet{chen2015graph} proposed a nonparametric approach based on the graphs to test the similarity between observations.  The aggregation technique testing whether there exists a gap between two time regions is also a popular way for change-point analysis. }
For example, \citet{jirak2015uniform} showed asymptotic properties by using the coordinate-wise statistic to test the single change point. { \citet{bai2010common} theoretically analyzed the consistency of the single-change point in panel data.} \citet{chen2021inference} improved the aggregation way of the data and explored to test the multiple change-points in high-dimensional time series. For the change-point estimation of complex high-dimensional data,
\citet{romanenkova2022indid} considered the change-point estimation of the complex sequential stream, video data, by representation learning for the deep model. \citet{chang2019kernel} combined the kernel two-sample test with recurrent neural networks (RNNs) to estimate the different types of change-points. They introduced an auxiliary generative model to generate an approximate distribution and use the maximum mean discrepancy to conduct the two-sample test, but there is no information about the location of the change-point. Although there are various change-point estimation techniques, fewer works directly focus on tracking the distribution of time series. 

\subsection{Our contributions}
Our main contributions are summarized as the following two aspects.
The first contribution is theoretically confirming that GANs are effective for estimating $\pi^l$, the joint distribution of $l$-tuple samples, of a stationary time series. More precisely, 
for the stationary time series whose dependence {  satisfies appropriate regularity assumptions}, we show that DNN-based GANs can be applied to learn its joint distribution, i.e., like the i.i.d. samples, the time series can be fed into the device of the neural networks of the GANs to obtain a stable generator $\hat{\boldsymbol {g}}$ such that $\hat{\boldsymbol {g}}(\boldsymbol z)$ has a distribution very close to $\pi^l$ ($\boldsymbol z$ is a Gaussian random variable). We prove a non-asymptotic error bound between the estimator of the measure $\pi^l$ and itself, which clearly demonstrates the estimator is consistent and thus effective. Our proof of this theoretical result is based on a typical block technique,  see {  [\citet{berkes2014komlos, liu2009strong, gouezel2010almost, lu2022amost}].}


By the validity of GANs for estimating the joint distribution of stationary times series, we put forward an algorithm based on GANs for estimating the position of the change point at the combination of two different stationary time series. {The idea is straightforward. Since the observed data on both sides of the change point follow two different stationary distributions, there should exist a significant difference between their DNN-based GAN estimators. Our second main result is to develop an algorithm which can efficiently capture this difference and thus obtain an estimation of the change point}. 

We implement two Monte Carlo experiments to validate our main results.
In the first experiment, for validating our theoretical results, we train a GAN model to learn a set of stationary sequences from the multivariate autoregressive  model (AR(3)) by rolling 5-sized window, and we compute the correlations and autocorrelations of the  generated 5-tuple samples. The results show that GAN can generate a set of strongly auto-correlated time series patches, and the generated sample patches have similar correlations as the real data. The second experiment demonstrates an example of our proposed single change-point estimation algorithm at the combination of two different stationary time series, in which the change point can be accurately estimated. We also apply the proposed algorithm to a real world empirical application about the human activity recognition dataset to estimate the volunteer's change of activities in the third experiment.

\subsection{Organization of the paper and some notations}
The paper is organized as follows. In Section \ref{sec2}, we provide some preliminary knowledge and assumptions for our main result. {  In Section \ref{sec3}, we give the convergence rate of GANs estimator for time series and the algorithm to estimate the change-point in a time series sequence. Three examples of simulation and real data analysis are given in Section \ref{sec4}. The proof of our main result is deferred to Section \ref{sec5}. }

We finish this section by introducing some notations which will be frequently used in
sequel.	Denote $[n]:=\{1, \ldots, n\}$. Given a vector $\mathbf{w} \in \mathbb{R}^d,\|\mathbf{w}\|$ will refer to the Euclidean norm, and for $h \geqslant 1,\|\mathbf{w}\|_h=(\sum_{i=1}^d|\mathbf{w}_i|^h)^{1 / h}$ will refer to the $\ell_h$ norm, $\|\mathbf w\|_\infty$ refer to the largest elements of $\mathbf w$. Given a matrix $\mathbf W\in \mathbb{R}^{p}\times \mathbb{R}^l$ and reals $h,l \geqslant 1$, we let $\|\mathbf W\|_{h,l}:=(\sum_k(\sum_j|\mathbf W_{j, k}|^h)^{l / h})^{1 / l}$ denote the $l$-norm of the $h$-norms of the columns of $\mathbf W$. The following defines the vectorization of the matrix
$${vec}\left(\left(\mathbf W_{i j}\right)_{1 \leq j \leq l}^{1 \leq i \leq p}\right)=\left(\mathbf W_{11}, \mathbf W_{21}, \ldots, \mathbf W_{p 1}, \mathbf W_{12}, \ldots, \mathbf W_{p 2}, \ldots, \mathbf W_{1 l}, \ldots, \mathbf W_{p l}\right)^{\top}\in \mathbb{R}^{lp}$$
For a real number $x\in\R$, we denote the integer part of $x$ by $\lfloor x\rfloor$ and the smallest integer number greater than $x$ by $\lceil x\rceil$.

For a real-valued function $f: \mathcal{X} \rightarrow \mathbb{R}$, we define $\|f\|_{\infty}:= \sup _{\boldsymbol{x} \in \mathcal{X}}|f(\boldsymbol{x})|$. 
Let $\operatorname{Lip}(\Omega, c)$ be all the real-valued Lipschitz functions from $\Omega$ to $\mathbb{R}$ with Lipschitz constant $c$, that is,
\begin{eqnarray*}
	\operatorname{Lip}(\Omega, c):=\left\{f: \Omega \rightarrow \mathbb{R}:\left|f\left(\boldsymbol{z}_1\right)-f\left(\boldsymbol{z}_2\right)\right| \leq c\left\|\boldsymbol{z}_1-\boldsymbol{z}_2\right\| \text { for any } \boldsymbol{z}_1, \boldsymbol{z}_2 \in \Omega\right\}.
\end{eqnarray*}

For $b>0$ with $b=r+q$, where $r\in\N_0:=\N\cup\{0\}$, $q\in(0,1]$, denote the H\"older class with smoothness index $b$ as 
\begin{eqnarray*}
	\mathcal H^b(\R^d):=\big\{f:\R^d\to\R,\max_{\|\alpha\|_1\le r} \|\partial^\alpha f\|_\infty\le1, \max_{\|\alpha\|_1= r}\sup_{\boldsymbol x\neq \boldsymbol y}\frac{|\partial^\alpha f(\boldsymbol x)-\partial^\alpha f(\boldsymbol y)|}{\|\boldsymbol x-\boldsymbol y\|^q}\le1   \big\},
\end{eqnarray*}
and we denote $\mathcal H=\mathcal H^b(\R^d)$ for the simplicity. For any subset $\mathcal{X} \subseteq \mathbb{R}^d$, we denote $\mathcal{H}^b(\mathcal{X}):=\left\{f: \mathcal{X} \rightarrow \mathbb{R}, f \in \mathcal{H}^b\left(\mathbb{R}^d\right)\right\}$.

Let $\mu$ and $\nu$ be a pair of probability measures on a space $\mathcal{X}$, and $\mathcal{F}$ be a class of functions $f: \mathcal{X} \rightarrow \mathbb{R}$ that are integrable with respect to  $\mu$ and $\nu$. Define integral probability metric as
\begin{eqnarray*}
	d_{\mathcal{F}}(\mu,\nu ):=\sup _{f \in \mathcal{F}}\left|\int f(\mathrm{d}\mu -\mathrm{d}\nu )\right|=\sup _{f \in \mathcal{F}}\left|\mathbb{E}_{X\sim \mu}[f(X)]-\mathbb{E}_{Z \sim \nu}[f(Z)]\right|.
\end{eqnarray*}

Wasserstein distance between two probability measures $\mu$ and $\nu$ is defined as
\begin{eqnarray*}
	\mathcal W_1(\mu, \nu):=\inf _{\gamma \in \Gamma(\mu, \nu)} \int_{\mathcal{X} \times \mathcal{X}} \|x-y\| \mathrm{~d} \gamma(x, y),
\end{eqnarray*}
where $\Gamma(\mu, \nu)$ denotes the collection of all measures on $\mathcal{X} \times \mathcal{X}$ with marginals $\mu$ and $\nu$ on the first and second factors respectively. 

By the duality theorem of Kantorovich and Rubinstein (1958), we have 
\begin{eqnarray*}
	\mathcal W_1(\mu, \nu)=d_{\text {Lip}(\mathcal{X},1)}(\mu,\nu )=\sup _{f \in \text {Lip}(\mathcal{X},1)} \left\{\int_\mathcal{X} f(x) \mathrm{d}(\mu-\nu)(x) \mid \text {continuous } f: \mathcal{X} \rightarrow \mathbb{R}\right\}.
\end{eqnarray*}

\section{Preliminaries}\label{sec2}
In this section, we first give the preliminaries about the assumptions on the stationary time series we want to study, then introduce the method of DNN-based generative adversarial networks (GANs) to estimate the joint distribution of this time series, and give the assumptions on the architectures of DNNs that we will use to model the generator and discriminator. 

\subsection{Time Series Assumptions}\label{sec: Assumptions}
\citet{wu2005nonlinear} proposed predictive dependence measures to provide another look at the fundamental issue of dependence. Using this tool, he established a theory for high-dimensional inferences under dependence. It is different from most of the current research on a high-dimensional inference that assumes the underlying observations are independent. The related concepts are as follows.

The $p$-dimensional stationary time series in \citet{wu2005nonlinear} is in the form of
\begin{eqnarray}\label{time series}
	\boldsymbol X_n :=\boldsymbol G(...,\boldsymbol \e_{n-1},\boldsymbol \e_n),
\end{eqnarray} 
where $\boldsymbol \e_n,n\in \Z$ are i.i.d. random variables and $\boldsymbol G$ is a measurable function such
that $\boldsymbol X_n$ is well-defined. 

Let $\left(\boldsymbol \varepsilon_i^{\prime}\right)$ be an i.i.d copy of $\left(\boldsymbol \varepsilon_i\right)$, $\boldsymbol \xi_i=\left(\ldots, \boldsymbol \varepsilon_{i-1}, \boldsymbol \varepsilon_i\right)$ and $\boldsymbol \xi_i^{\prime}=(\ldots, \boldsymbol \varepsilon_{i-1}^{\prime}, \boldsymbol \varepsilon_i^{\prime})$. For any set $I \subset \mathbb{Z}$, $\boldsymbol \varepsilon_{j, I}=\boldsymbol \varepsilon_j^{\prime}$ if $j \in I$ and $\boldsymbol \varepsilon_{j, I}=\boldsymbol \varepsilon_j$ if $j \notin I$; $\boldsymbol \xi_{i, I}=\left(\ldots,\boldsymbol  \varepsilon_{i-1, I},\boldsymbol  \varepsilon_{i, I}\right)$. 

The predictive dependence measure is defined as
\begin{eqnarray}
	[\mathbb{E}(\|\mathbb{E}(\boldsymbol X_n \mid\boldsymbol  \xi_0)-\mathbb{E}(\boldsymbol X_n \mid\boldsymbol  \xi_{0, I})\|^\beta)]^{1 /\beta},
\end{eqnarray} 
where $n\geq 0$, $\beta \geq 1$. 

Under the assumption that $\{\boldsymbol X_n\}_{n\in\N}$ satisfies  geometric moment contraction, i.e.
\begin{eqnarray}
	\E\|\boldsymbol X_n-\boldsymbol G(\boldsymbol \xi_0',\boldsymbol \e_1,...,\boldsymbol \e_n)\|^\beta\le C_1e^{-C_2n},
\end{eqnarray}
\citet{wu2007strong} extended KMT  approximation from independent random variables to a large class of dependent stationary processes. \citet{berkes2014komlos} extended this result to optimal rate.

The time series studied in this paper satisfies the following assumption:

\begin{assumption}\label{assum1}
	We assume that the time series $\{\boldsymbol X_n\}_{n\in\N}$ in the form of \eqref{time series} satisfies geometric moment contraction, that is,
	\begin{eqnarray}\label{e:GMC}
		\E\|\boldsymbol X_n-\boldsymbol G(\boldsymbol \xi_0',\boldsymbol \e_1,...,\boldsymbol \e_n)\|^\beta\le C_1e^{-C_2n},
	\end{eqnarray}
	where $C_1$ and $C_2$ are positive constants, $n\geq 0$, $\beta \geq 1$, $\boldsymbol \xi_0=(...,\boldsymbol \e_{-1},\boldsymbol \e_0)$, $\boldsymbol \xi_0'=(...,\boldsymbol \e_{-1}',\boldsymbol \e_0')$ and $\{\boldsymbol \e_n'\}_{n\in\Z}$ is an i.i.d. copy of $\{\boldsymbol \e_n\}_{n\in\Z}$.
\end{assumption}
We further assume that the time series satisfies either Assumption \ref{assum2} or Assumption \ref{assum4}. 
\begin{assumption}\label{assum2}
	The stationary measure $\pi$ of time series $\{\boldsymbol X_n\}_{n\in\N}$ is sub-Gaussian, that is, there exist $\boldsymbol m\in\R^p$ and $v>0$ such that, for all $\boldsymbol \alpha\in\R^p$, 
	\begin{eqnarray*}
		\E_{\boldsymbol x\sim\pi}[\exp\{\boldsymbol \alpha^T(\boldsymbol x-\boldsymbol m)\}]\le e^{\|\boldsymbol \alpha\|^2v^2/2},	
	\end{eqnarray*}
	where $T$ is the transpose operator. Without loss of generality, we assume $\boldsymbol{m}=\boldsymbol{0}$.
\end{assumption}

{
\begin{assumption}\label{assum4}
	The stationary measure $\pi$ of time series $\{\boldsymbol X_n\}_{n\in\N}$ has $\omega$-th moment bound with $\omega\ge2$, that is,
	\begin{eqnarray*}
		\E_{\boldsymbol x\sim\pi}[\|\boldsymbol x\|^\omega]\le C_\omega.	
	\end{eqnarray*}
\end{assumption}	
}

\subsection{GAN for Estimating Joint Distribution of Time Series}
Suppose that data of the time series $\boldsymbol X_n :=\boldsymbol G(...,\boldsymbol \e_{n-1},\boldsymbol \e_n)$ in the form of \eqref{time series} for all periods in the sample is available. In this subsection, we introduce the method to estimate the joint distribution of this stationary time series with length $l$ using GANs (we theoretically prove the effectiveness of this approach in Section 3). Specifically, we divide the data of time series into multiple small blocks, each containing data for $l$ time points, then design a method of DNN-based GANs to estimate the joint distribution using the data from all blocks. The detailed idea is as follows.

We adopt the chunking method for time series in rolling-window analysis, which makes the most efficient use of samples and is useful for subsequent experimental analysis. We set the size of a rolling window, i.e., the number of consecutive observation per rolling window as $l$, and the number of increments between successive rolling windows as 1 period, then partition the entire data set of time series into $\bar n=n-l+1$ subsamples. Hence, the $i$ th block of time series in this rolling time window can be written as
$$
\{\boldsymbol X_{i}, \boldsymbol X_{i+2},...,\boldsymbol X_{i+(l-1)}\},1\le i\le \bar n
$$
We denote $vec(\{\boldsymbol X_{i}, \boldsymbol X_{i+1},...,\boldsymbol X_{i+(l-1)}\})$ as $\boldsymbol X_i^\prime$, and the joint distribution of $\boldsymbol X_i^\prime$ as $\pi^l$. Obviously, the stationary measure $\pi$ of time series is just $\pi^1$. Here $\pi^l$ can be a high-dimensional distribution, which is difficult to estimate with traditional  methods. 

Let $\mathcal G$ and $\mathcal D$ be the collections of neural networks. Let $\boldsymbol g\in \mathcal G$ be the generator function from $\R$ to $\R^{lp}$ and $d\in\mathcal D$ be the discriminator function from $\R^{lp}$ to $\R$. GANs for time series can be formulated as the following optimal problem,
\begin{eqnarray}\label{e:GANs}
	\boldsymbol g^{*} \in \arg \min _{\boldsymbol g \in \mathcal{G}} \max _{d \in \mathcal{D}}\left\{\E_{\boldsymbol x \sim \pi^l}[d( \boldsymbol x)]-\E_{\boldsymbol z \sim \nu}[d(\boldsymbol g(\boldsymbol z))]\right\},
\end{eqnarray}
where $\nu$ is the Gaussian distribution defined on the one-dimensional space $\R$ to generate fake signal. For convenience, $\mathcal{D}$ is assumed in symmetric class, which means if $d \in \mathcal{D}$, then $-d \in \mathcal{D}$. For any subset $\mathcal{X} \subseteq \mathbb{R}^{lp}$, we denote $\mathcal{D}(\mathcal{X}):=\left\{f: \mathcal{X} \rightarrow \mathbb{R}, f \in \mathcal{D}\left(\mathbb{R}^{lp}\right)\right\}$. The empirical version of \eqref{e:GANs} is considered as
\begin{eqnarray}\label{e:epGANs}
	\boldsymbol {\hat g} \in \arg \min _{\boldsymbol g \in \mathcal{G}} \max _{d \in \mathcal{D}}\Big\{ \frac1{\bar n}\sum_{i=1}^{\bar n}d(\boldsymbol X_i^\prime)-\frac1m\sum_{i=1}^m d(\boldsymbol g(\boldsymbol z_i))\Big\},
\end{eqnarray}
where $\bar n=n-l+1$, $\{\boldsymbol z_i\}_{i=1,...,m}$ are i.i.d. random variables with distribution $\nu$. 

\subsection{DNN Model Assumptions}
We consider the feed-forward DNN as the L-compositional function class indexed by parameter 
\begin{eqnarray}\label{e:DNNpara}
	\boldsymbol \Theta=(\boldsymbol W_0,\boldsymbol W_1,...,\boldsymbol W_L),
\end{eqnarray}
where $\boldsymbol W_l\in\R^{D_{l+1}\times D_l}$, for $l=0,1,...,L$, $D_0,...,D_{L+1} \in\N$. $L$ and $N:=\max\{D_1,...,D_L\}$ are depth and width of the DNNs respectively. 

We denote the DNN with an input $\boldsymbol x \in \R^d$ by
\begin{eqnarray}\label{e:DNN}
	\mathcal {NN}_{d}(N,L):=\{f(\boldsymbol x;\boldsymbol \Theta)=\boldsymbol W_L\sigma_L(... \boldsymbol W_1\sigma_1(\boldsymbol W_0\boldsymbol x)...)\in \mathbb{R}^{D_{L+1}}|\boldsymbol \Theta\}.
\end{eqnarray}
When the input dimension $d$ is clear from contexts, we simply denote it by $\mathcal{NN}(N,L)$. 

We consider that the neural networks $\mathcal{N N}(N, L)$ in \eqref{e:DNN} has ReLU activation functions 
$$\left\{\sigma_j(\boldsymbol{z})=\left(z_1 \vee 0, \cdots, z_{D_j} \vee 0\right) \text{ for } \boldsymbol{z}:=\left(z_1, \ldots, z_{D_j}\right)^{T} \in \mathbb{R}^{D_j}\right\}_{j=1}^L.$$ 
\begin{assumption}\label{assum3}
	We further denote $\mathcal{NN}(N,L,K)$ by the function of $\mathcal{NN}(N,L)$ with 
	$$
	\prod_{l=0}^{L}\|\boldsymbol W_l\|\le K,
	$$
	where  $K>0$, $\|\cdot\|$ is the operator norm of a matrix.
\end{assumption}

\begin{remark}
	{  Since ReLu function is a Lipschitz function with Lipschitz constant $1$, it is easy to see that for any function $\boldsymbol f\in \mathcal{NN}(N,L,K)$, its Lipschitz constant is smaller than $\prod_{l=0}^{L}\|\boldsymbol W_l\|$ which is bounded by $K$, that is,
		$$||\boldsymbol f(\boldsymbol x)-\boldsymbol f(\boldsymbol y)||\le K||\boldsymbol x-\boldsymbol y||.$$
		For more details, see \citet[Section 3]{jiao2022approximation}.
	}
	
\end{remark}

\section{Main Results}\label{sec3}
{  In this section, two main results are introduced, the first result is a a non-asymptotic error bound between the joint measure $\pi^l$ of the time series and its estimator $\boldsymbol {\hat g}_{\#}\nu$ which is defined below. Based on the  theoretical result, the second main result for estimating the single change-point in time series sequence is also introduced.
	\subsection{Non-asymptotic bound for the estimator}
Recall that $\pi^l$ is the joint distribution of each small block with length $l$ of the time series $\boldsymbol X_n :=\boldsymbol G(...,\boldsymbol \e_{n-1},\boldsymbol \e_n)$ in the form of \eqref{time series}. Our goal is to establish a non-asymptotic error bound of the GANs estimator for joint measure $\pi^l$ based on integral probability metric, 
	\begin{eqnarray*}
		d_{\mathcal H}(\pi^l, \boldsymbol {\hat g}_{\#}\nu)=\sup_{h\in\mathcal H}\big\{\E_{\boldsymbol x \sim \pi^l}[h(\boldsymbol x)]-\E_{\boldsymbol z \sim \nu}[h(\boldsymbol {\hat g}(\boldsymbol z))]\big\},
	\end{eqnarray*}
	where $\mathcal H=\mathcal H(\R^{lp})$, $\boldsymbol {\hat g}_\#\nu$ is the push-forward measure defined as $\boldsymbol {\hat g}_\#\nu(A)=\nu(\boldsymbol {\hat g}^{-1}(A))$ for measurable set $A\subset \R^{lp}$. For ease of notation, we use $\mathbb{E}_\gamma[f]$ instead of $\mathbb{E}_{\boldsymbol{x} \sim \gamma}[f(\boldsymbol{x})]$ for any distribution $\gamma$ and function $f$ below.

The following theorem is our first main result, whose proof will be postponed to the last section.
\begin{theorem}[Non-asymptotic Bound]\label{thm:main}
	Let Assumptions \ref{assum1} and \ref{assum3} hold true, and let either Assumption \ref{assum2} or Assumption \ref{assum4} be satisfied. Let  $\nu$ be a probability measure on $\mathbb{R}$ which is absolutely continuous with respect to the Lebesgue measure. Recall \eqref{e:epGANs} and assume that 
$\E |\boldsymbol {z}_1|<\infty$. Let the evaluation class be $\mathcal{H}=\mathcal{H}^b\left(\mathbb{R}^{lp}\right)$ with $b=r+q,r \in \mathbb{N}_0$ and $q \in(0,1]$. 
	
	Then, for the given sample size $n$, there exists a generator family
	$$
	\mathcal{G}=\left\{\boldsymbol g \in \mathcal{N} \mathcal{N}\left(N_{\mathcal{G}}, L_{\mathcal{G}},K_{\mathcal{G}}\right): \boldsymbol g(\mathbb{R}) \subseteq[-a_n,a_n]^{lp}\right\},
	$$ 
	with $N_{\mathcal{G}} \geq 7 lp+1, L_{\mathcal{G}} \geq 2$ and $a_n$ to be determined later, and a discriminator family
	$$
	\mathcal{D}=\{d\in\mathcal{N} \mathcal{N}\left(N_{\mathcal{D}}, L_{\mathcal{D}}, K_{\mathcal{D}}\right): d(\mathbb{R}^{lp}) \subseteq[-B,B]\}\cap \operatorname{Lip}\left(\mathbb{R}^{lp}, C\right), 
	$$
	with some $B>0$ and $C>0$, $N_{\mathcal{D}} \geq c\left(K_{\mathcal{D}} / \log ^\gamma K_{\mathcal{D}}\right)^{(2 lp+b) /(2 lp+2)}, L_{\mathcal{D}} \geq 4 \gamma+2$, $\gamma:=\left\lceil\log _2(lp+r)\right\rceil$,  such that
	$$
	n\leq \frac{N_{\mathcal{G}}-lp-1}{2}\left\lfloor\frac{N_{\mathcal{G}}-lp-1}{6l p}\right\rfloor\left\lfloor\frac{L_{\mathcal{G}}}{2}\right\rfloor+1+l ,
	$$
	such that GAN estimator $\boldsymbol{\hat{g}}$, defined by \eqref{e:epGANs}, satisfies the following results.
	
	{     
	(i). As Assumption \ref{assum2} holds and $a_n=\log (lpn)$, i.e.,
	$$
	\mathcal{G}=\left\{\boldsymbol g \in \mathcal{N} \mathcal{N}\left(N_{\mathcal{G}}, L_{\mathcal{G}},K_{\mathcal{G}}\right): \boldsymbol g(\mathbb{R}) \subseteq[-\log (lpn),\log (lpn)]^{lp}\right\},
	$$ 
	\begin{eqnarray*}
		d_{\mathcal H}(\pi^l, \boldsymbol {\hat g}_{\#}\nu):&=&\sup_{h\in\mathcal H}\big\{\E_{\boldsymbol x \sim \pi^l}[h(\boldsymbol x)]-\E_{\boldsymbol z \sim \nu}[h(\boldsymbol {\hat g}(\boldsymbol z))]\big\}\\
		&\le& O\left((\log lpn)^b({  K_{\mathcal{D}}/\log^\gamma K_{\mathcal{D}}})^{-\frac{b}{lp+1}}\right)+O\left(lK_{\mathcal{D}} e^{-C_2 n^{\alpha}l}\right)+O\left(B\sqrt{{lt_1}/{n^{1-\alpha}}}\right)\\
		&&+O\left(l\sqrt{{1}/{n}} K_{\mathcal{D}} \sqrt{L_{\mathcal{D}}+\log lp}\right)+O\left(K_{\mathcal{D}} K_{\mathcal{G}} \sqrt{L_{\mathcal{D}}+L_{\mathcal{G}}}/{m} +\sqrt{{t_2}/{m}}\right),
	\end{eqnarray*}	
	with probability at least $1-2l{e}^{-t_1}-le^{-C_2n^{\alpha}l/2\beta}-e^{-t_2}$, for $t_1,t_2>0$, $0< \alpha< 1/2$.

	(ii). As Assumption \ref{assum4} holds and $a_n=n^{\frac{1-\alpha}{2\omega}}$, i.e.,
	$$
	\mathcal{G}=\left\{\boldsymbol g \in \mathcal{N} \mathcal{N}\left(N_{\mathcal{G}}, L_{\mathcal{G}},K_{\mathcal{G}}\right): \boldsymbol g(\mathbb{R}) \subseteq[-n^{\frac{1-\alpha}{2\omega}},n^{\frac{1-\alpha}{2\omega}}]^{lp}\right\},
	$$  
		\begin{eqnarray*}
			d_{\mathcal H}({\pi^l}, \boldsymbol {\hat g}_{\#}\nu)   
			&\le& O\left(n^{\frac{(1-\alpha)}{2\omega}}({  K_{\mathcal{D}}/\log^\gamma K_{\mathcal{D}}})^{-\frac{b}{lp+1}}\right)+O\left(lK_{\mathcal{D}} e^{-C_2 n^{\alpha}l}\right)+O\left(Bl\sqrt{{t_1}/{n^{1-\alpha}}}\right)\\
			&&+O\left(l\sqrt{{1}/{n}} K_{\mathcal{D}} \sqrt{L_{\mathcal{D}}+\log lp}\right)+O\left(K_{\mathcal{D}} K_{\mathcal{G}} \sqrt{L_{\mathcal{D}}+L_{\mathcal{G}}}/{m} +\sqrt{{t_2}/{m}}\right).
		\end{eqnarray*}
with probability at least $1-2l{e}^{-t_1}-le^{-C_2sl/2\beta}-e^{-t_2}$, for $t_1,t_2>0$, $0< \alpha< 1/2$.
		$1-2l{e}^{-t_1}-le^{-C_2sl/2\beta}-e^{-t_2}$
}
	
\end{theorem}
{  
	\begin{remark}
		The error of estimating the marginal distribution with i.i.d. sample data by GANs is given as $n^{-b/p}\vee n^{-1/2}\log^{c(b,p)}n$, where $c(b,p)$ is a positive constant depending on the index $b$ of H\"older class and the dimension $p$, see [\citet{huang2022error}]. Considering the time series case with Assumption \ref{assum2}, we further assume that $K_{\mathcal D}=C n^{\frac{1}{2}-\frac{b}{2(p+1+b)}}$, $m\ge Cn K_{\mathcal G}\sqrt L_{\mathcal G} $ and $B\le (\log pn)^b$, for $l=1$ and large enough sample size $n$, we can get 
		\begin{eqnarray*}
			d_{\mathcal H}(\pi^1, \boldsymbol {\hat g}_{\#}\nu)
			&\le& O\left((\log pn)^b n^{-\frac{b}{2(p+1+b)}}\right)+O\left(Bn^{-\frac12(1-\alpha)}\right),
		\end{eqnarray*}	
		with high probability. Here $\pi^1$ is the stationary distribution $\pi$ of time series, $\alpha\in(0,\frac12)$ is the parameter in the blocking technique to handle the dependence issue of time series. For any $p$ and $b$, one can choose $\alpha<\frac{p+1}{p+1+b}$ such that
		\begin{eqnarray*}
			d_{\mathcal H}(\pi, \boldsymbol {\hat g}_{\#}\nu)
			&\le& O\left((\log pn)^b n^{-\frac{b}{2(p+1+b)}}\right),
		\end{eqnarray*}	
		which is close to the error of of i.i.d. case as $p\ge 2b$, i.e. $n^{-b/p}$.
		
		{      Considering the time series case with Assumption \ref{assum4}, we further assume that $K_{\mathcal D}=C n^{\frac{(p+1)(2-2\alpha+\omega)}{2\omega(p+1+b)}}$, $m\ge Cn K_{\mathcal G}\sqrt L_{\mathcal G} $ and $B\le (\log pn)^b$, $\omega>p$ for $l=1$ and large enough sample size $n$, we can get 
		\begin{eqnarray*}
			d_{\mathcal H}(\pi^1, \boldsymbol {\hat g}_{\#}\nu)
			&\le& O\left((\log pn)^b n^{-\frac{b}{2(p+1+b)}+\frac{(1-\alpha)(p+b)}{2\omega(p+1+b)}}\right)+O\left(Bn^{-\frac12(1-\alpha)}\right),
		\end{eqnarray*}	
		with high probability. Here $\pi^1$ is the stationary distribution $\pi$ of time series, $\alpha\in(0,\frac12)$ is the parameter in the blocking technique to handle the dependence issue of time series. For any $p$ and $b$, one can choose $\alpha<\frac{\omega p+\omega+p+b}{\omega(p+1+b)+p+b}$ such that
		\begin{eqnarray*}
			d_{\mathcal H}(\pi, \boldsymbol {\hat g}_{\#}\nu)
			&\le& O\left((\log pn)^b n^{-\frac{b}{2(p+1+b)}+\frac{p+b}{2\omega(p+1+b)}}\right).
		\end{eqnarray*}	
	
	}
		
	\end{remark} 
}	

\subsection{Change-point estimation for time series}\label{sec:change-point}
The result of Theorem \ref{thm:main} has a {     wide} range of potential applications. In this subsection, we apply it to estimate change point in time series distributions. {     In the following analysis, we will focus our discussion on the case of a structural change in the marginal distribution, corresponding to $l=1$ in Theorem \ref{thm:main}. The procedure and associated analysis can be easily extended to the estimation of change point in the joint distributions by considering $l>1$.}

Let $\boldsymbol X = \{\boldsymbol X_{1},...,\boldsymbol X_{n}\}$ be a $p$-dimensional time series sequence {     with a structural break in the distribution at time $\tau$. Thus,} there exists a $\tau \in \mathbb N$ such that $\{\boldsymbol X_{1},...,\boldsymbol X_{\tau}\}$ is stationary with a stationary distribution $\pi$ and that $\{\boldsymbol X_{\tau+1},...,\boldsymbol X_{n}\}$ is also stationary but with a distribution $\pi'$. This form of time series is called the time series with a single change point [\citet{chen2015graph}]. 
Estimating the change point $\tau$ is an important research topic in time series {     applications in statistics and econometrics}. 

{     Based on the theoretical result in the previous subsection, we  propose an algorithm for estimating the single change point in time series using GANs, see Algorithm \ref{alg:the_alg} below. It is reasonable to assume that the change point $\tau$ is bounded away from the beginning or the end, combining with the following $K$-block estimation procedure,  we need to assume $K<\tau<n-K$. Similar assumptions are also required in the literature of change point estimations.}

The {     procedure} for estimating $\tau$ contains three stages. In the \textbf{first} stage, we divide the time series $\boldsymbol X$ into $\lceil n/K\rceil$ blocks, each having $K (1 \ll K \ll n)$ samples, as the following: 
$\boldsymbol{B}_1,...,\boldsymbol{B}_{\lceil n/K \rceil}$, where $\boldsymbol{B}_i=\{\boldsymbol X_{(i-1)K+1},...,\boldsymbol X_{iK}\}$  for each $i$. Then, we train a stable GAN model by the sample points in the first block $\boldsymbol B_1$, i.e., we obtain a generator $\boldsymbol {g}$ and a discriminator $d$. At the \textbf{second} stage, we use the samples in the remaining blocks $\{\boldsymbol B_{2},...,\boldsymbol B_{\lceil n/K \rceil}\}$ to compute the loss values, i.e., for every $\boldsymbol{B}_i$, we compute
\begin{equation}\label{eq:lossvalue}
\mathcal{L}_i=\frac{1}{K} \sum_{\boldsymbol X_j \in \mathbf{B}_{i}} d(\boldsymbol X_j)-\frac{1}{K}\sum_{j=1}^{K} d(\boldsymbol {g}(\boldsymbol z_j)), \quad \text{for} \quad i=2,...,\lceil n/K \rceil,
\end{equation}
where $\boldsymbol z_1,...,\boldsymbol z_K$ are i.i.d. standard normal distributed random variables, (here, if the sample size $h$ of the last block $\boldsymbol B_{\lceil n/K \rceil}$ is less than $K$, we use $h$ to substitute $K$ in (\ref{eq:lossvalue})).  
Obviously, for the blocks that contain the samples following the same distribution as the ones used to train GAN, the loss values would be much smaller than the ones containing the samples with different distributions. The $\mathcal{L}_k$ of the block in which the change point is located will be very different from that of its neighboring blocks. {  Motivated by these properties}, we can compute the difference, $|\mathcal{L}_{i+1}-\mathcal{L}_i|$, between the neighboring two blocks and find some $\boldsymbol{B}_{k}$ and $\boldsymbol{B}_{k+1}$ that have the largest difference on the loss values. After determining the two candidate blocks $\boldsymbol{B}_{k}$, $\boldsymbol{B}_{k+1}$, we merge them into one group, $\boldsymbol{G}=\{\boldsymbol{B}_k,\boldsymbol{B}_{k+1}\}=\{\boldsymbol X_{(k-1)K+1},...,\boldsymbol X_{(k+1)K}\}$. This operation can help us narrow down the time series data set and further facilitate the estimation of change point. In the \textbf{third} stage, for each sample $\boldsymbol X_j \in \boldsymbol G$ for $j=(k-1)K+1,...,(k+1)K$, we build $\boldsymbol H_j=\{\boldsymbol X_{j-w},...,\boldsymbol X_j,...\boldsymbol X_{j+w}\}$ centered on the sample $\boldsymbol X_j$  and with radius $w(1\ll w\ll K)$ to do the estimation for $(k-1)K+1+w\le j\le (k+1)K-w$. To guarantee $\boldsymbol H_j$ is well-defined, we define that $\boldsymbol H_j=\{\boldsymbol X_{(k-1)K+1},...,\boldsymbol X_j,...\boldsymbol X_{j+w}\}$ if $j <(k-1)K+1+w$ and $\boldsymbol H_j=\{\boldsymbol X_{j-w},...,\boldsymbol X_j,...\boldsymbol X_{(k+1)K}\}$ if $j >(k+1)K-w$. All the samples in $\boldsymbol H_j$ are used to compute 
$$
	\dot{\mathcal{L}}_j=\frac{1}{2w+1} \sum_{\boldsymbol X_i\in \boldsymbol H_{j}} d(\boldsymbol X_i)-\frac{1}{2w+1}\sum_{i=1}^{2w+1}d(\boldsymbol g(\boldsymbol z_i)) \quad \text{for} \quad  j=(k-1)K+1,...,(k+1)K.
$$

where $\boldsymbol z_1,...,\boldsymbol z_{2w+1}$ are i.i.d. standard normal distributed random variables, (here, if the sample size $\kappa$ of $\boldsymbol H_{j}$  is less than $2w+1$, we use $\kappa$ to substitute $2w+1$).
At the change point (e.g., $\boldsymbol X_{j^*}$), there exists an obvious difference between the loss of two adjacent windows, $\dot{\mathcal{L}}_{j^*}$ and $\dot{\mathcal{L}}_{j^*-1}$,  because the window $\boldsymbol H_{j^*}$ contains more samples following different distribution than $\boldsymbol H_{j^*-1}$. Thus, we compute the difference $\dot{\mathcal{L}}_{j}-\dot{\mathcal{L}}_{j-1}$ for $j=(k-1)K+2,...,(k+1)K$ between two adjacent loss values to locate the change point $\boldsymbol X_{j^*}$ where the index satisfies
$$j^* =\arg\max [\dot{\mathcal{L}}_{j}-\dot{\mathcal{L}}_{j-1},j=(k-1)K+2,...,(k+1)K].
$$

\begin{algorithm}[]
	\caption{Change-point estimation}
	\label{alg:the_alg}
	\KwIn{ $n$: sample size; $K$: the number of samples in each block;  $w$: is the radius of the rolling window; }
	Sample a $p$-dimensional time series sequence $\boldsymbol{X}=\{\boldsymbol X_{1},...,\boldsymbol X_{n}\}$;\\
	Divide $\boldsymbol X$ into $\lceil n/K\rceil$ blocks, $\boldsymbol{B}_1,...,\boldsymbol{B}_{\lceil n/K\rceil}$ with $\boldsymbol{B}_i=\{\boldsymbol X_{(i-1)K+1},...,\boldsymbol X_{iK}\}$, for each $i$.\\
	$\textbf{Stage 1: Training a stable discriminator and generator}$\\
	\hspace*{-0.15in}\quad  \textbf{Initialize} $\quad \theta_d, \theta_g \leftarrow$ Gaussian; \\
	Use the samples in $\boldsymbol B_1$ to
	train a discriminator $d$;\\
	Sample i.i.d noise $\{\boldsymbol z_1,...,\boldsymbol z_{K}\}$ from the standard normal distribution to train a generator $\boldsymbol g$;\\
	\textbf{Stage 2: Select the blocks possibly including the change point}\\
	\For{$i = 2,...,\lceil n/K \rceil$}{
		{ Input $\boldsymbol{X}_j \in \mathbf{B}_{i}$ into $d$};\\
		Sample $K$ i.i.d noises $\{\boldsymbol z_{1},...,\boldsymbol z_{K}\}$ from the the standard normal distribution;\\
		{ Input $\{\boldsymbol z_{1},...,\boldsymbol z_{K}\}$ into $ \boldsymbol g$;} \\
		{ Compute $\mathcal{L}_i=\frac{1}{K} \sum_{\boldsymbol X_j \in \mathbf{B}_{i}} d(\boldsymbol X_j)-\frac{1}{K}\sum_{j=1}^{K} d(g(\boldsymbol z_j))$;}\\
		{Store $\mathcal{L}_i$. }\\
		
	}
	Compute the distances $|\mathcal{L}_{i+1}-\mathcal{L}_i|$ for $i=2,..., \lceil n/K \rceil-1$, and find the $\boldsymbol{B}_{k}$ and $\boldsymbol{B}_{k+1}$ blocks having the largest distance. Then merge $\boldsymbol{B}_{k}$ and $\boldsymbol{B}_{k+1}$ into one group $\boldsymbol{G}=\{\boldsymbol{B}_k,\boldsymbol{B}_{k+1}\}=\{\boldsymbol X_{(k-1)K+1},...,\boldsymbol X_{(k+1)K}\}$.  \\
	
	$\textbf{Stage 3: Estimate the location of change point}$\\
	 Build $\boldsymbol H_j=\{\boldsymbol X_{j-w},...,\boldsymbol X_j,...\boldsymbol X_{j+w}\}$ for $(k-1)K+1+w\le j\le (k+1)K-w$, $\boldsymbol H_j=\{\boldsymbol X_{(k-1)K+1},...,\boldsymbol X_j,...\boldsymbol X_{j+w}\}$ for $j <(k-1)K+1+w$ and $\boldsymbol H_j=\{\boldsymbol X_{j-w},...,\boldsymbol X_j,...\boldsymbol X_{(k+1)K}\}$ for $j >(k+1)K-w$. \\
	\For{$j=(k-1)K+1+w\le j\le (k+1)K-w$}{
		{Input samples $\boldsymbol H_j$ in $\boldsymbol G$ into $d$;}\\
		Sample $2w+1$ i.i.d noises $\{\boldsymbol z_{1},...,\boldsymbol z_{2w+1}\}$ from the standard normal distribution;\\
		{ Input $\{\boldsymbol z_{1},...,\boldsymbol z_{2w+1}\}$ into $g$;} \\
		{ Compute $\dot{\mathcal{L}}_j=\frac{1}{2w+1} \sum_{\boldsymbol X_i\in \boldsymbol H_{j}} d(\boldsymbol X_i)-\frac{1}{2w+1}\sum_{i=1}^{2w+1}d(\boldsymbol g(\boldsymbol z_i))$;}\\
		{Store $\dot{\mathcal{L}}_j$. }
	}
	Compute the distances $\dot{\mathcal{L}}_{j}-\dot{\mathcal{L}}_{j-1}$ for $j=(k-1)K+2+w\le j\le (k+1)K-w$, and record the $j^{*}$th index having the largest distance.
\end{algorithm}

\section{Monte Carlo and Empirical Applications}\label{sec4}
\subsection{Estimating the joint distribution of a multivariate autoregressive series}\label{sec:experiment1}

In this section, we illustrate the effectiveness of GANs for modeling the stationary time series data { and for estimating the joint distribution} from the multivariate autoregressive model as follows:
\begin{equation}\label{eq:AR(3)}
	\boldsymbol{x}_t = \boldsymbol{x}_{t-1}-0.01\boldsymbol{x}_{t-2}-0.5\boldsymbol{x}_{t-3}+\boldsymbol{v}_t,
\end{equation}
where $\boldsymbol{x}_t\in \mathbb{R}^{20}$ and $\boldsymbol{v}_t\sim \boldsymbol{N}(\boldsymbol{0},0.2\boldsymbol{I}_{20})$ for {  $t=1,...,10500$}. The initial values are set as $\boldsymbol{x}_1=\boldsymbol{x}_2=\boldsymbol{0}$, and {  each element of $\boldsymbol{x}_3$ is independently sampled from uniform distribution $U(0,0.01)$}, the seed is set 3 (Python language, $\mathsf{random.seed}(3)$). Considering that the sequences at initial times are nonstationary, we remove the first 500 observations to guarantee the stationarity of the training data and retain the remaining 10000 observations to train GAN. { We adopt the general training method of GAN (Algorithm 1 in \citet{arjovsky2017wasserstein}) to train the discriminator and generator. As we describe in Section \ref{sec:change-point}, the GAN model is trained by rolling window on the time series. We set the size of window to be 5, i.e., 5 time slices $\{\boldsymbol{x}_t,...,\boldsymbol{x}_{t+4}\}$ constituting a sample with size $5\times20$, fed into the discriminator.} The neural network of the discriminator for GAN is composed of nine-linear layers with the LeakyReLu activation function. The structure of the generator contains two convolution layers and five deconvolution layers, in which each layer is also followed by a LeakyReLu activation function. Specifically, the two convolutional layers first transform the Gaussian input noises into latent variables, and then the subsequent deconvolution layers upsample the latent variables to produce fake samples with the same dimension as the real ones. In the training process, we train the generator one time once the discriminator is trained 5 times. The batch size is set as 128. After training a stable generator, we feed a set of random Gaussian noises into the generator and then obtain a 5-tuple 20-dimensional generated samples. Figure \ref{fig:plot1} { shows the  correlations of the generated samples of 5-tuple generated samples (Figure \ref{fig:plot1a}) and their autocorrelations  (Figure \ref{fig:plot1b}). It can be seen that the correlations of sample patches generated by GAN are very close to that of the real data. We randomly select two dimensions (the second and fifth dimensions) in the generated sample patches and plot the autocorrelations. We notice that the generated sample patches have strong autocorrelations. The results reveal that GAN can effectively estimate the joint distribution of stationary time series data.}
\begin{figure}[] 
	\centering 
	\subfigure[]{
		\includegraphics[width=0.8\linewidth]{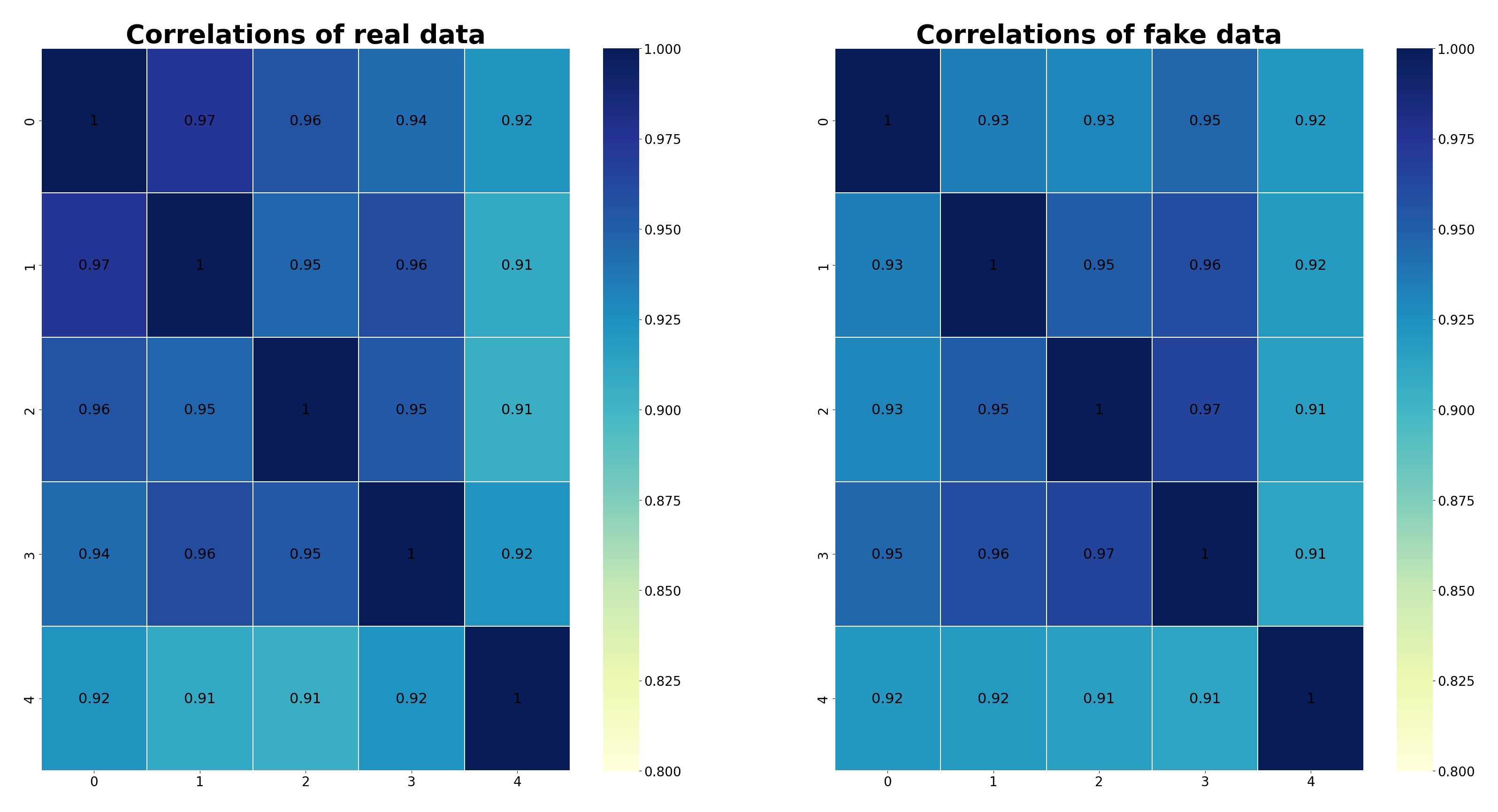}\label{fig:plot1a} }
	\subfigure[]{
		\includegraphics[width=0.45\linewidth]{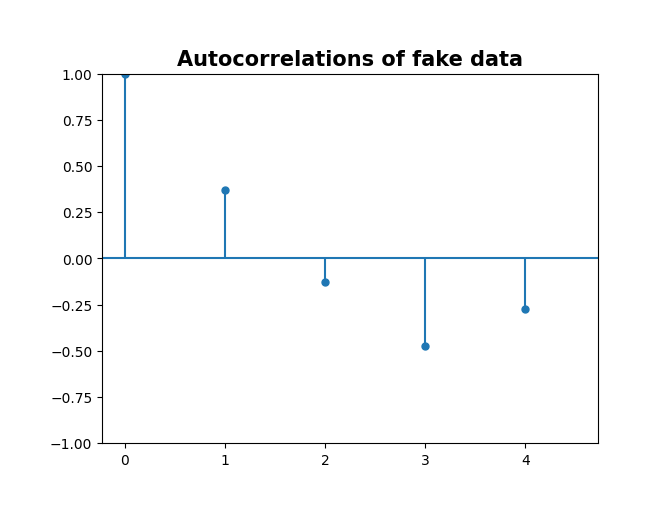} 
	\includegraphics[width=0.45\linewidth]{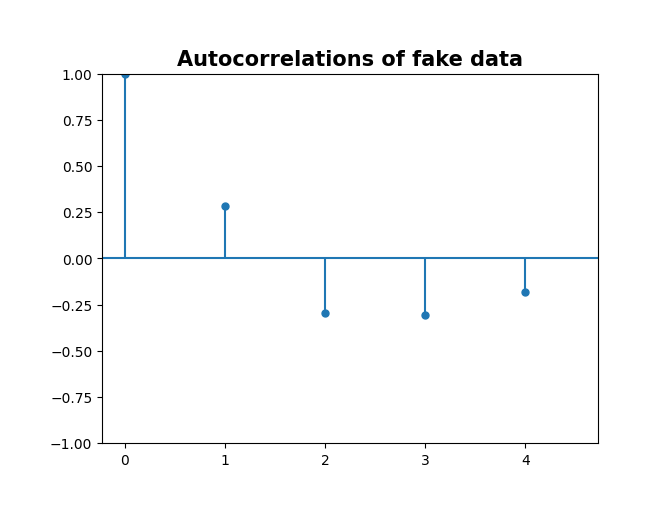} \label{fig:plot1b}}
	\caption{(a): The comparison of the correlations. (b): The autocorrelations of the second and fifth dimensions of the 5-tuple generated samples.}
	\label{fig:plot1}
\end{figure}
\begin{figure}[!t] 
	\centering 
	
	\includegraphics[width=0.6\linewidth]{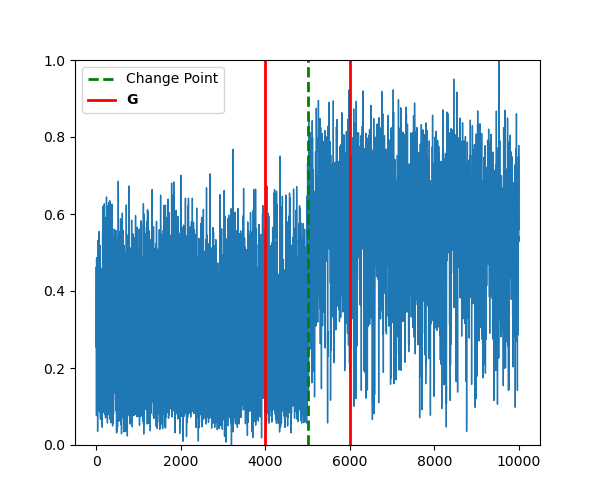}
	
	\caption{Loss values of 10000 stationary samples $\{\Tilde{\boldsymbol{y}}_t\}_{t=1}^{10000}$. The samples located in the interval $[4000,6000]$ are set as the elements of $\boldsymbol G$ (in \textbf{Stage 2}).}
	\label{Total}
\end{figure}
 
\begin{figure}[!t] 
	\centering 
	
	\includegraphics[width=0.8\linewidth]{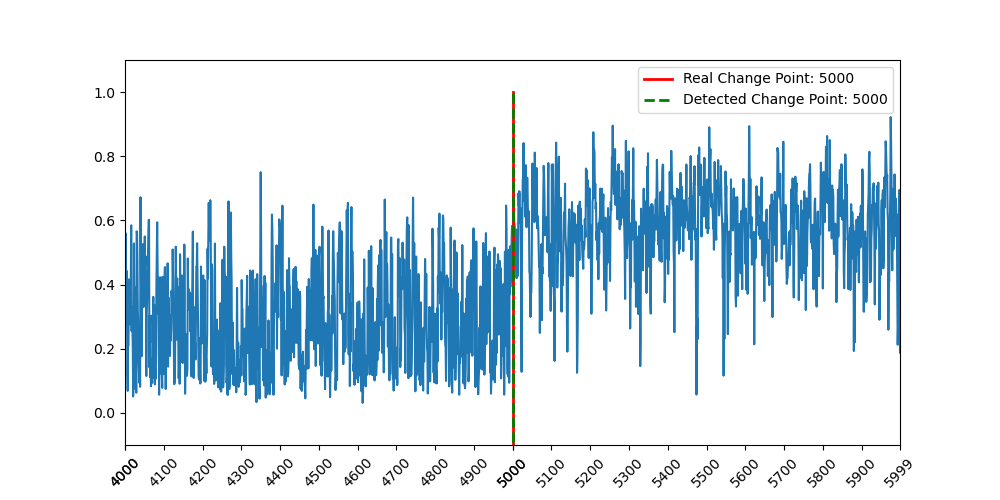}
	
	\caption{Loss values of each samples in $\boldsymbol G$ corresponding to the interval $[4000,6000]$ in Figure 
		\ref{Total}. The change point is accurately estimated.}
	\label{fig:plot2}
\end{figure}
\subsection{An example of the Algorithm \ref{alg:the_alg}}
We now present an example of change-point estimation for time series so as to illustrate the usefulness of our proposed estimation algorithm. {  Specifically, we generate a sequence of 23-dimensional 10500 time-series samples $\{\boldsymbol{y}_k\}_{k=1}^{10500}$ with a single change point, where $\boldsymbol y_k$ is generated by 
	$$
	\boldsymbol{y}_k=0.6\boldsymbol{y}_{k-1}-0.5\boldsymbol{y}_{k-2}+\boldsymbol{\epsilon}_k,
	$$
	and $\boldsymbol{\epsilon}_k$ is a Gaussian vector from $\boldsymbol{N}(\boldsymbol{0},\boldsymbol{I}_{23})$ for $k<5500$, $\boldsymbol{N}(\boldsymbol{2},\boldsymbol{I}_{23})$ for $k\ge 5500$. The initial values $\boldsymbol{y}_1=\boldsymbol{0}$ and $\boldsymbol{y}_2=\boldsymbol{0}$. We remove the first 500 non-stationary observations to obtain a sequence of stationary time series $\{\Tilde{\boldsymbol{y}}_t\}_{t=1}^{10000}$. So the change point exists at $t= 5000$ in time series $\{\Tilde{\boldsymbol{y}}_t\}_{t=1}^{10000}$.} We use our proposed change-point estimation algorithm to estimate the change point. Firstly, we split the 10000 samples into 10 blocks, each block having 1000 samples. Then we utilize the $1000$ samples in the first block to train a stable generator and discriminator (in \textbf{Stage 1}). Next, following the \textbf{stage 2}, we compute the loss values of the remaining blocks based on the trained generator and discriminator and find the two blocks having the largest loss values, i.e., the $5$th and $6$th blocks (in the time interval $[4000,6000]$), see Figure \ref{Total}. It can be seen that the loss values have obvious variation on the both sides of the change point. We merge the $5$th and $6$th blocks into group $\boldsymbol G$ and set the radius as $w=29$ (in \textbf{Stage 3}) to finally search the change point on $\boldsymbol G$. As shown in Figure \ref{fig:plot2}, the change point in $\boldsymbol G$ can be accurately estimated.
\subsection{Real data applications}
In previous Monte Carlo experiments, we demonstrated the outperformance of GAN on estimating the change point in the simulated data based on the Algorithm \ref{alg:the_alg}.  In this section, we apply our proposed algorithm to a real-world dataset, the human activity recognition (HAR) dataset (\citet{anguita2013public}). The HAR dataset contains the activity recognition data of 30 volunteers over times (sec) from the accelerometer, gyroscope, magnetometer and GPS sensors when they carry a waist-mounted smartphone. The recognition of  activity includes six different activities, such as sitting, walking, driving, etc.  We only use the gyroscope data to analyze. The gyroscope data records 3-axial angular velocity at a constant rate of 50Hz, which can be regarded as three dimensions. We extract 10800 observations of one volunteer in the gyroscope data to estimate the change of activities, i.e., the variations of the volunteer's activity recognition from inactive to walking. Figure \ref{fig:realline} shows the extracted time series on three dimensions. It can be seen that there exists an obvious change between the two activities at the sides of the 9468th observation. The observations before the 9468th sample indicate that the state of this volunteer is inactive. For the observations after 9468th sample, their distribution is significantly distinct with the former ones, which represents that the state of activity changes to walking.   

Next, we use our proposed algorithm to estimate the location of this change point. Here, we set the size of the rolling window as 5 to reduce the influence of abnormal fluctuation of the observations after the 9468th sample on the loss values and enhance the estimate accuracy. We split the 10800 observations into 5 blocks, then, we use the same training methods in the first experiments (section \ref{sec:experiment1}) to learn the joint distribution of the samples in the first block by rolling 5-sized windows and obtain a stable discriminator and generator (\textbf{Stage 1}). Following the \textbf{Stage 2}, we compute the loss values of the remaining blocks and find the two blocks corresponding to the largest difference in the loss values, see Figure \ref{fig:plot3} (a). Then, we search the change point in the interval $[6480,10799]$ with the radius $w=4$. Figure \ref{fig:plot3} (b) shows that the estimated location of the change point is the same as the real one. Therefore, we can conclude that our proposed change-point estimation  algorithm can accurately estimate the volunteer's activity change.
\begin{figure}[] 
	\centering 
		\includegraphics[width=0.85\linewidth]{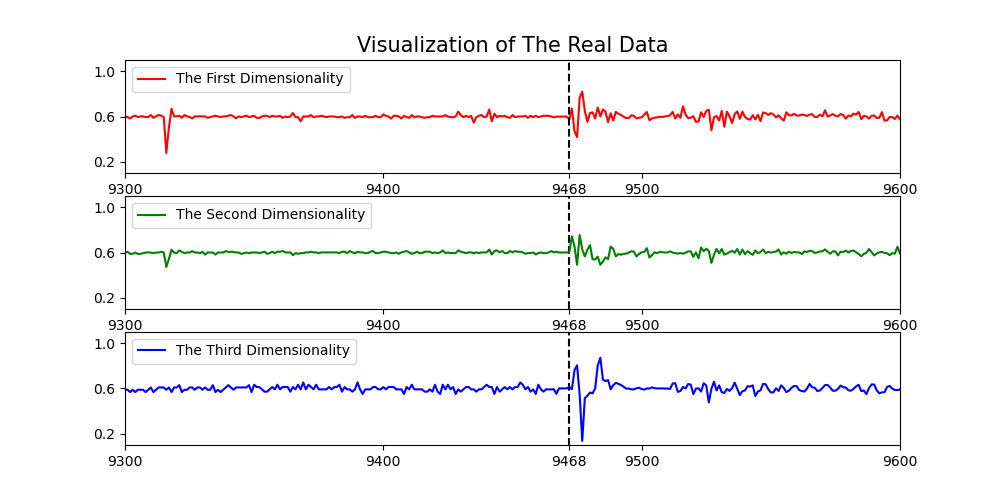} 
	\caption{Visualization of the gyroscope data near the change point.}
	\label{fig:realline}
	
\end{figure}
\begin{figure}[] 
	\centering 
	\subfigure[]{
		\includegraphics[width=0.95\linewidth]{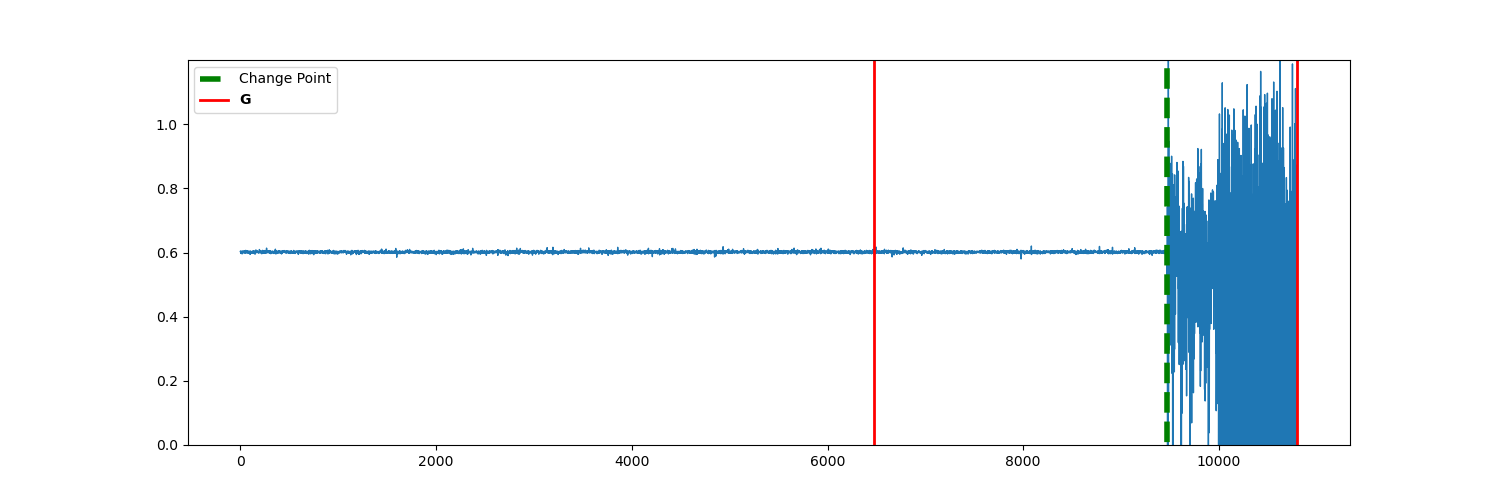} }
	\subfigure[]{
		\includegraphics[width=0.95\linewidth]{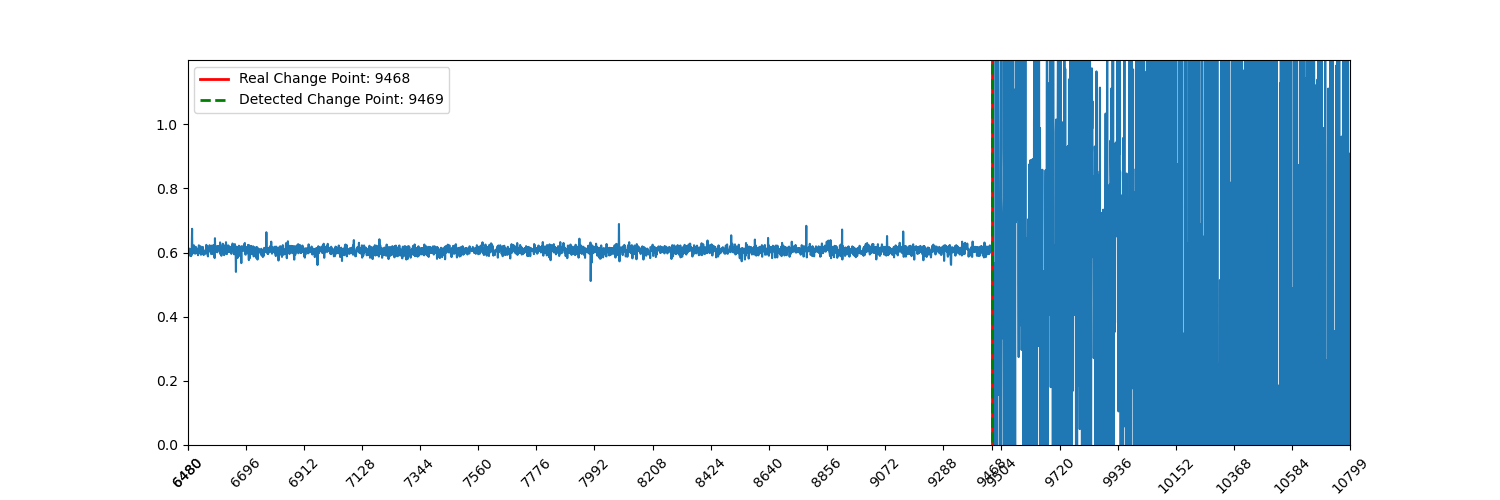} }
	\caption{(a): Loss values of the extracted gyroscopes data, the samples located in the interval $[6480,10799]$ are set as the elements of $\boldsymbol G$ (in \textbf{Stage 2}). (b): Loss values of each samples in $\boldsymbol G$, the change point is accurately estimated.}
	\label{fig:plot3}
\end{figure}

\section{Proof of Theorem \ref{thm:main}}\label{sec5}
In this section, we prove Theorem \ref{thm:main}, based on a decomposition for $d_{\mathcal H}(\pi^l,\boldsymbol {\hat g}_{\#}\nu)$ and the estimates for each term in the decomposition.  

\subsection{Auxiliary lemmas}
In order to prove Theorem \ref{thm:main}, we need the following auxiliary lemmas. The first lemma is an error decomposition for $d_{\mathcal H}(\pi^l,\boldsymbol {\hat g}_{\#}\nu)$, while the others are estimates for each term in this decomposition. 
\subsubsection{Error decomposition of $d_{\mathcal H}(\pi^l,\boldsymbol {\hat g}_{\#}\nu)$}  
\begin{lemma}[Error Decomposition]\label{lem:decom}
	{     
	Suppose $\boldsymbol g_\#\nu$ is supported on
	$\left[-a_n,a_n\right]^{lp}$ for all $\boldsymbol g \in \mathcal{G}$ and assume $\mathcal{D}$ is in symmetric class. Let $\boldsymbol{\hat{g}}$ be defined in \eqref{e:epGANs} and $\mathcal{H}^b$ be function class defined on $\mathbb{R}^{lp}$ with $h\in \mathcal{H}^b$ and $b=r+q$.

		(i). Suppose Assumption \ref{assum2} is satisfied, there exists $c >0$ such that for any 
		$$
		\mathcal{D}=\{d\in\mathcal{N} \mathcal{N}\left(N_{\mathcal{D}}, L_{\mathcal{D}}, K_{\mathcal{D}}\right): d(\mathbb{R}^{lp}) \subseteq[-B,B] ,B<\infty\},
		$$
		with $N_{\mathcal{D}}\ge c(K_{\mathcal{D}}/\log^\gamma K_{\mathcal{D}})^{(2lp+b)/(2lp+2)}$,$L_{\mathcal{D}}\ge 4\gamma+4$, $\gamma:=\lceil\log_2(lp+r)\rceil$, one has
		\begin{eqnarray*}
			d_{\mathcal H}(\pi^l, \boldsymbol {\hat g}_{\#}\nu)
			&\le& 2(2a_n)^b({  K_{\mathcal{D}}/\log^\gamma K_{\mathcal{D}}})^{-\frac b{lp+1}}+2(1+B) lp e^{-a_n+v^2 / 2}\\
			&&+\sup_{d\in\mathcal D}\Big\{\E_{\pi^l}[d]-\frac1{\bar n}\sum_{i=1}^{\bar n}d(\boldsymbol X_i^\prime)\Big\}+\inf_{\boldsymbol g\in\mathcal G}\sup_{d\in\mathcal D}\Big\{ \frac1{\bar n}\sum_{i=1}^{\bar n}d(\boldsymbol X_i^\prime)-\E_{\boldsymbol g_\#\nu}[d]\Big\}\\
			&&+2\sup_{d\circ \boldsymbol g\in\mathcal D\circ\mathcal G}\{\E_{\hat\nu_m}[d\circ\boldsymbol  g]-\E_{\nu}[d\circ \boldsymbol g]\}.
		\end{eqnarray*}
		where $\bar n=n-l+1$, $\hat{\nu}_m:=\frac{1}{m} \sum_{i=1}^m \delta_{\boldsymbol{z}_i}$.
		
		(ii).	Suppose Assumption \ref{assum4} is satisfied, there exists $c >0$ such that for any $\mathcal{D}$ defined in (i) above,
		\begin{eqnarray*}
			d_{\mathcal H}(\pi^l, \boldsymbol {\hat g}_{\#}\nu)
			&\le& 2(2a_n)^b({  K_{\mathcal{D}}/\log^\gamma K_{\mathcal{D}}})^{-\frac b{lp+1}}+C_\omega(1+B) l a_n^{-\omega}\\
			&&+\sup_{d\in\mathcal D}\Big\{\E_{\pi^l}[d]-\frac1{\bar n}\sum_{i=1}^{\bar n}d(\boldsymbol X_i^\prime)\Big\}+\inf_{\boldsymbol g\in\mathcal G}\sup_{d\in\mathcal D}\Big\{ \frac1{\bar n}\sum_{i=1}^{\bar n}d(\boldsymbol X_i^\prime)-\E_{\boldsymbol g_\#\nu}[d]\Big\}\\
			&&+2\sup_{d\circ \boldsymbol g\in\mathcal D\circ\mathcal G}\{\E_{\hat\nu_m}[d\circ\boldsymbol  g]-\E_{\nu}[d\circ \boldsymbol g]\}.
		\end{eqnarray*}

}
	
\end{lemma}
\begin{proof}
	(i). For any $\e>0$, it is easy to know that there exist an $h_\e\in\mathcal H$ and a $d_\e\in\mathcal D$ such that
	\begin{eqnarray*}
		\sup_{h\in\mathcal H}\{\E_{\pi^l}[h]-\E_{\boldsymbol{\hat {g}}_\#\nu}[h]\}\le \E_{\pi^l}[h_\e]-\E_{\boldsymbol{\hat {g}}_\#\nu}[h_\e]+\e,
	\end{eqnarray*}
	and 
	\begin{eqnarray*}
		\|h_\e-d_\e\|_\infty\le \inf_{d\in\mathcal D}\|h_\e-d\|_{\infty}+\e.	
	\end{eqnarray*}	
	Then, we have
	\begin{eqnarray}\label{e:decom1}
		&&\sup_{h\in\mathcal H}\{\E_{\pi^l}[h]-\E_{\boldsymbol{\hat {g}}_\#\nu}[h]\}\nonumber\\
		&\le&  \E_{\pi^l}[h_\e]-\E_{\boldsymbol{\hat {g}}_\#\nu}[h_\e]+\e\nonumber\\
		&=& \E_{\pi^l}[h_\e]-\E_{\pi^l}[d_\e]+\E_{\pi^l}[d_\e]-\E_{\boldsymbol{\hat {g}}_\#\nu}[d_\e]+\E_{\boldsymbol{\hat {g}}_\#\nu}[d_\e]-\E_{\boldsymbol{\hat {g}}_\#\nu}[h_\e]+\e\nonumber\\
		&=& \{\E_{\pi^l}[h_\e-d_\e]+\E_{\boldsymbol{\hat {g}}_\#\nu}[d_\e-h_\e]\}+\{\E_{\pi^l}[d_\e]-\E_{\boldsymbol{\hat {g}}_\#\nu}[d_\e]\}
		+\e.
	\end{eqnarray}
	
	Step 1: We first deal with the first terms of \eqref{e:decom1}. Let $A_n=[-a_n,a_n]^{lp}$, it is easy to see that
	\begin{eqnarray}\label{e:error1}
		&&\E_{\pi^l}[h_\e-d_\e]+\E_{\boldsymbol {\hat g}_\#\nu}[d_\e-h_\e]\\
		&=& \E_{\pi^l}[(h_\e-d_\e)1_{A^c_n}]+\E_{\pi^l}[(h_\e-d_\e)1_{A_n}]+\E_{\boldsymbol {\hat g}_\#\nu}[d_\e-h_\e].\nonumber
	\end{eqnarray}
	For the first term of \eqref{e:error1}, the boundedness of $h_\e$, $d_\e$ and sub-Gaussian assumption yield
	\begin{eqnarray}\label{e:3.3}
		\E_{\pi^l}[(h_\e-d_\e)1_{A^c_n}]
		&\le& \|h_\e-d_\e\|_\infty\E_{\pi^l}[1_{A_n^c}]\\
		&\le& (1+\|d\|_\infty+\e)\mathbb{P}_{\boldsymbol{x} \sim {\pi^l}}\left(\|\boldsymbol{x}\|_{\infty} \geq a_n\right)\nonumber \\
		&\le& (1+\|d\|_\infty+\e)\sum_{i=1}^{lp} \mathbb{P}_{\boldsymbol{x} \sim {\pi^l}}\left(\left|x_i\right| \geq a_n\right)\nonumber \\
		&\le& 2(1+B+\e) lp e^{-a_n+v^2 / 2} \nonumber.
	\end{eqnarray}
	Since $\boldsymbol{g}(\R)\subset A_n$, one can get the estimate of last two terms of \eqref{e:error1}
	\begin{eqnarray*}
		\mathbb{E}_{\pi^l}\left[\left(h_{\varepsilon}-d_{\varepsilon}\right) 1_{A_n}\right]+\mathbb{E}_{\boldsymbol{\hat{g}} \# \nu}\left[d_{\varepsilon}-h_{\varepsilon}\right] \leq 2 \sup _{h \in \mathcal{H}\left(A_n\right)} \inf _{d \in \mathcal{D}\left(A_n\right)}\|h-d\|_{\infty}+2 \varepsilon.
	\end{eqnarray*} 
	
	\citet[Theorem 3.2]{jiao2022approximation} guarantees that, for any $\bar h\in \mathcal{H}^b([0,1]^{lp}), b=r+q$, there exists $c >0$ such that for any $\bar d\in \mathcal{NN}(N,L,K)$, $N\ge c(K/\log^\gamma K)^{(2lp+b)/(2lp+2)}$ and $L\ge 4\gamma+2$, $\gamma:=\lceil\log_2(lp+r)\rceil$, 
	\begin{eqnarray*}
		\sup_{\bar h\in\mathcal H([0,1]^{lp})}\inf_{\bar d\in\mathcal {NN}(N,L,K)}\|\bar h-\bar d\|_\infty\le (K/ \log^\gamma K)^{-\frac b{lp+1}}.
	\end{eqnarray*}
	Let $h(x){  1_{A_n}}=(2a_n)^b\bar h(\frac{x}{2a_n}+\frac12)$, $d_0(x){  1_{A_n}}=(2a_n)^b\bar d(\frac{x}{2a_n}+\frac12)$, and
	\begin{eqnarray*}
		d(x){  1_{A_n}}=(d_0(x)\land B)\lor(-B)=\sigma(d_0(x)+B)-\sigma(d_0(x)-B)-B.
	\end{eqnarray*}
	Then, we have
	\begin{eqnarray*}
		\sup _{h \in \mathcal{H}\left(A_n\right)} \inf _{d \in \mathcal{D}\left(A_n\right)}\|h-d\|_\infty&=&\sup _{\bar h \in \mathcal{H}\left([0,1]^{lp}\right)} \inf _{d \in \mathcal{D}\left(A_n\right)}\|(2a_n)^b\bar h(\frac{x}{2a_n}+\frac12)-d(x)\|_{\infty}\\
		&\le& (2a_n)^b\sup_{\bar h\in\mathcal H([0,1]^{lp})}\inf_{\bar d\in\mathcal D\left([0,1]^{lp}\right)}\|\bar h-\bar d\|_\infty\le(2a_n)^b(K_{\mathcal{D}}/\log^{\gamma}K_{\mathcal{D}})^{-\frac b{lp+1}},
	\end{eqnarray*}
	where $d\in \mathcal{NN}(N_{\mathcal{D}},L_{\mathcal{D}},K_{\mathcal{D}})$, with  $N_{\mathcal{D}}\ge c(K_{\mathcal{D}}/\log^\gamma K_{\mathcal{D}})^{(2lp+b)/(2lp+2)}$ and $L_{\mathcal{D}}\ge 4\gamma+4$, $\gamma:=\lceil\log_2(lp+r)\rceil$. Thus, we obtain
	\begin{eqnarray*}
		\E_{\pi^l}[h_\e-d_\e]+\E_{\boldsymbol {\hat g}_\#\nu}[d_\e-h_\e]
		&\le& 2(2a_n)^b(K_{\mathcal{D}}/\log^{\gamma}K_{\mathcal{D}})^{-\frac b{lp+1}}+ 2(1+B+\e) lp e^{-a_n+v^2 / 2}+2\e.
	\end{eqnarray*}	
	
	Step 2: For the rest terms of \eqref{e:decom1}, that is $\E_{\pi^l}[d_\e]-\E_{\boldsymbol{\hat {g}}_\#\nu}[d_\e]$, one considers the empirical process and gets
	\begin{eqnarray*}
		\E_{\pi^l}[d_\e]-\E_{\boldsymbol{\hat {g}}_\#\nu}[d_\e]
		&=&\E_{\pi^l}[d_\e]-\frac1{\bar n}\sum_{i=1}^{\bar n}d_\e(\boldsymbol X_i^\prime)+\frac1{\bar n}\sum_{i=1}^{\bar n}d_\e(\boldsymbol X_i^\prime)-\frac1m\sum_{i=1}^{m}d_\e(\boldsymbol {\hat g}(\boldsymbol z_i))\\
		&&+\frac1m\sum_{i=1}^{m}d_\e(\boldsymbol {\hat g}(\boldsymbol z_i))
		-\E_{\boldsymbol \nu}[d_\e(\boldsymbol{\hat {g}}(\boldsymbol z))]\\
		&\le&\sup_{d\in\mathcal D}\Big\{\E_{\pi^l}[d]-\frac1{\bar n}\sum_{i=1}^{\bar n}d(\boldsymbol X_i^\prime)\Big\}+\inf_{\boldsymbol g\in\mathcal G}\sup_{d\in\mathcal D}\Big\{ \frac1n\sum_{i=1}^{\bar n}d(\boldsymbol X_i^\prime)-\frac1m\sum_{i=1}^m d(\boldsymbol g(\boldsymbol z_i))\Big\}\\
		&&+\sup_{d\circ \boldsymbol g\in\mathcal D\circ\mathcal G}\{\E_{\hat\nu_m}[d\circ\boldsymbol  g]-\E_{\nu}[d\circ \boldsymbol g]\}\\
		&\le&\sup_{d\in\mathcal D}\Big\{\E_{\pi^l}[d]-\frac1{\bar n}\sum_{i=1}^{\bar n}d(\boldsymbol X_i^\prime)\Big\}+\inf_{\boldsymbol g\in\mathcal G}\sup_{d\in\mathcal D}\Big\{ \frac1{\bar n}\sum_{i=1}^{\bar n}d(\boldsymbol X_i^\prime)-\E_{\boldsymbol g_\#\nu}[d]\Big\}\\
		&&+\inf_{\boldsymbol g\in\mathcal G}\sup_{d\in\mathcal D}\Big\{\E_{\boldsymbol g_\#\nu}[d]-\frac1m\sum_{i=1}^m d( \boldsymbol g(\boldsymbol z_i))\Big\}+\sup_{d\circ \boldsymbol g\in\mathcal D\circ\mathcal G}\{\E_{\hat\nu_m}[d\circ \boldsymbol g]-\E_{\nu}[d\circ \boldsymbol g]\},
	\end{eqnarray*}
	where the first inequality follows the definition of $\boldsymbol {\hat g}$ and $\hat\nu_m:=\frac1m\sum_{i=1}^{m}\delta_{\boldsymbol z_i}$ is the empirical measure. Combining the inequality above with \eqref{e:decom1} and a straight calculation implies
	\begin{eqnarray*}
		d_{\mathcal H}({\pi^l}, \boldsymbol {\hat g}_{\#}\nu)
		&\le& 2(2a_n)^b(K_{\mathcal{D}}/\log^\gamma K_{\mathcal{D}})^{-\frac b{lp+1}}+2(1+B+\e) lp e^{-a_n+v^2 / 2}+3\e\\
		&&+\sup_{d\in\mathcal D}\Big\{\E_{\pi^l}[d]-\frac1{\bar n}\sum_{i=1}^{\bar n}d(\boldsymbol X_i^\prime)\Big\}+\inf_{\boldsymbol g\in\mathcal G}\sup_{d\in\mathcal D}\Big\{ \frac1{\bar n}\sum_{i=1}^{\bar n}d(\boldsymbol X_i^\prime)-\E_{\boldsymbol g_\#\nu}[d]\Big\}\\
		&&+2\sup_{d\circ \boldsymbol g\in\mathcal D\circ\mathcal G}\{\E_{\hat\nu_m}[d\circ\boldsymbol  g]-\E_{\nu}[d\circ \boldsymbol g]\}.
	\end{eqnarray*} 
	Let $\e\to0$ and the proof is complete. 
	
	{     
(ii). The proof is similar with (i) except the calculation of \eqref{e:3.3} in Step 1. For the first term of \eqref{e:error1},  Assumption \ref{assum4} and the Markov inequality imply	
\begin{eqnarray}\label{e:3.3(1)}
	\E_{\pi^l}[(h_\e-d_\e)1_{A^c_n}]
	&\le& \|h_\e-d_\e\|_\infty\E_{\pi^l}[1_{A_n^c}]\\
	&\le& (1+\|d\|_\infty+\e)\mathbb{P}_{\boldsymbol{x} \sim {\pi^l}}\left(\|\boldsymbol{x}\|_{\infty} \geq a_n\right)\nonumber \\
	&\le& (1+B+\e)\E_{\boldsymbol{x} \sim {\pi^l}}[\max_{1\le i\le l}\max_{1\le j\le p}|\boldsymbol{x}_{i,j}|^\omega]a_n^{-\omega}\nonumber\\
	&\le& l(1+B+\e)\E_{\boldsymbol{x} \sim {\pi}}[\|\boldsymbol{x}\|^\omega] a_n^{-\omega}\nonumber\\
	&\le& lC_\omega (1+B+\e)a_n^{-\omega}.\nonumber
\end{eqnarray}
Combining equality \eqref{e:3.3(1)} and Step 2 above, we obtain
	\begin{eqnarray*}
	d_{\mathcal H}({\pi^l}, \boldsymbol {\hat g}_{\#}\nu)
	&\le& 2(2a_n)^b(K_{\mathcal{D}}/\log^\gamma K_{\mathcal{D}})^{-\frac b{lp+1}}+lC_\omega(1+B+\e)a_n^{-\omega}+3\e\\
	&&+\sup_{d\in\mathcal D}\Big\{\E_{\pi^l}[d]-\frac1{\bar n}\sum_{i=1}^{\bar n}d(\boldsymbol X_i^\prime)\Big\}+\inf_{\boldsymbol g\in\mathcal G}\sup_{d\in\mathcal D}\Big\{ \frac1{\bar n}\sum_{i=1}^{\bar n}d(\boldsymbol X_i^\prime)-\E_{\boldsymbol g_\#\nu}[d]\Big\}\\
	&&+2\sup_{d\circ \boldsymbol g\in\mathcal D\circ\mathcal G}\{\E_{\hat\nu_m}[d\circ\boldsymbol  g]-\E_{\nu}[d\circ \boldsymbol g]\}.
\end{eqnarray*} 
Let $\e\to0$ and the proof is complete.

}
	
\end{proof}
\subsubsection{Statistical error of the GAN estimator}
Recall that $\boldsymbol X_i:=\boldsymbol G(...,\boldsymbol{\e}_i)$, which means for any $i,j \in \mathbb{Z}$, $\boldsymbol X_i,\boldsymbol X_j$ is not independent. We denote the $sl$-dependence form of $\boldsymbol X_i$ by $\tilde{\boldsymbol X}_i=\E[\boldsymbol X_i|\mathcal F_{i-sl}^i]$ and $\mathcal F_{i-sl}^i=\sigma(\boldsymbol\e_{i-sl},...,\boldsymbol\e_i)$. Thus $\tilde{\boldsymbol X}_i$ only depends on $\{\boldsymbol\e_{i-sl},...,\boldsymbol\e_i\}$. We denote 
$\{\tilde{\boldsymbol X}_{i},\tilde{\boldsymbol X}_{i+1},...,\tilde{\boldsymbol X}_{i+(l-1)}\}$ as $\tilde{\boldsymbol X}_i^\prime$, it is easy to know that $\tilde{\boldsymbol X}_i^\prime$ is independent on $\tilde{\boldsymbol X}_{i-l(s+1)}^\prime$.

Let the stationary distribution of time series $\{\tilde{\boldsymbol X}_i\}$ be $\tilde{\pi}$, and the joint distribution of $\{\tilde{\boldsymbol X}_i^\prime\}$ be $\tilde{{\pi^l}}$. In general, we choose $n$ large enough that for some $s,l\in\mathbb N_+$, such that $sl\ll n$. 

Statistical error of the GAN estimator, $\sup_{d\in\mathcal D}\{\E_{\pi^l}[d]-\frac1{\bar n}\sum_{i=1}^{\bar n}d({\boldsymbol X}_i^\prime)\}$, can be controlled by the concentration for suprema of the GAN estimator as in following Lemma.
\begin{lemma}[Concentration for Supremum of GAN Empirical Processes]\label{lem:concentration}
	Suppose Assumption \ref{assum1}, Assumption \ref{assum2} or \ref{assum4} are satisfied and
	$$
	\mathcal{D}=\{d\in\mathcal{N} \mathcal{N}\left(N_{\mathcal{D}}, L_{\mathcal{D}}, K_{\mathcal{D}}\right): d(\mathbb{R}^{lp}) \subseteq[-B,B] ,B<\infty\},
	$$
	then with probability at least $1-2l\mathrm{e}^{-t_1}-le^{-C_2sl/2\beta}$
	\begin{eqnarray*}
		&&\sup _{d \in \mathcal{D}}\left|\frac{1}{\bar n} \sum_{i=1}^{\bar  n}\left(d\left({\boldsymbol X}_i^\prime\right)-\mathbb{E}_{{\pi^l}}\left[d\left({\boldsymbol X}_i^\prime\right)\right]\right)\right|\\
		&\le&\frac{1}{\bar n}\left( nlC_1^{1/\beta}K_{\mathcal{D}}e^{-C_2(sl+1)/\beta}/(1-e^{-C_2/\beta})+nlC_1^{1/\beta}K_{\mathcal{D}}e^{-C_2sl/(2\beta)-C_2/\beta}/({1-e^{-C_2/\beta}})\right)\\
		&&+\frac{1}{\bar n}C_3^\prime\left[l\sqrt{2sl+n}  K_{\mathcal{D}} \cdot \sqrt{L_{\mathcal{D}}+1+\log (lp)}+B \sqrt{t_1snl}\right].
	\end{eqnarray*}
where $\bar n=n-l+1$.
\end{lemma}
In order to prove Lemma \ref{lem:concentration}, we need to prove following lemma \ref{lem:difdepen} first. 

We can divide $\{\boldsymbol X_i^\prime\}_{1\leq i\leq n-l+1}$ into the following $l$ parts, that is
\begin{eqnarray*}
	\{\boldsymbol X_k^\prime, \boldsymbol X_{k+l}^\prime, ..., \boldsymbol X_{k+(Q_k-1)l}^\prime\},1\leq k\leq l
\end{eqnarray*}
where {$Q_k=\lfloor n/l\rfloor$, if $k\le (n \text{ mod } l)+1$; $Q_k=\lfloor n/l\rfloor-1$, if $(n\text{ mod } l)+1< k\leq l$.} In this way, there is no overlapping components between any two vectors in $\{\boldsymbol X_k^\prime, \boldsymbol X_{k+l}^\prime, ..., \boldsymbol X_{k+Q_kl}^\prime\}$. 

\begin{lemma}\label{lem:difdepen}
	Suppose Assumption \ref{assum1} is satisfied and
	$$
	\mathcal{D}=\{d\in\mathcal{N} \mathcal{N}\left(N_{\mathcal{D}}, L_{\mathcal{D}}, K_{\mathcal{D}}\right): d(\mathbb{R}^{lp}) \subseteq[-B,B] ,B<\infty\}.
	$$
	(i) We have  
	\begin{eqnarray*}
		\sup _{d \in \mathcal{D}}\left|\sum_{i=1}^{Q_k}\left(\mathbb{E}_{\pi^l}\left[d\left(\boldsymbol X_{k+l(i-1)}^\prime\right)\right]-\mathbb{E}_{\tilde{{\pi^l}}}\left[d\left(\tilde{\boldsymbol X}_{k+l(i-1)}^\prime\right)\right]\right)\right| \leq nC_1^{1/\beta}K_{\mathcal{D}}e^{-C_2(sl+1)/\beta}/(1-e^{-C_2/\beta}).
	\end{eqnarray*}
	(ii) With probability at least $1-e^{-C_2sl/2\beta}$, we have
	\begin{eqnarray*}
		\sup_{d\in\mathcal D}\big\{\sum_{i=1}^{Q_k}d(\boldsymbol X_{k+l(i-1)}^\prime)-\sum_{i=1}^{Q_k}d(\tilde{\boldsymbol X} _{k+l(i-1)}^\prime)  \big\}
		\le  nC_1^{1/\beta}K_{\mathcal{D}}e^{-C_2sl/(2\beta)-C_2/\beta}/({1-e^{-C_2/\beta}}).
	\end{eqnarray*}		
\end{lemma}
\begin{proof}
	(i) For the convenience of the mark, we denote $\boldsymbol G(\boldsymbol\xi_{i},\boldsymbol\e_{i+1},...,\boldsymbol\e_{j})$ as $\boldsymbol G_i^j$, and $\boldsymbol G(\boldsymbol\xi'_{i},\boldsymbol\e_{i+1},...,\boldsymbol\e_{j})$ as $\boldsymbol (\boldsymbol G')_i^j$. A straight calculation yields that
	\begin{eqnarray*}
		&&d(\boldsymbol X_i^\prime)-d(\tilde{\boldsymbol X}_i^\prime)\\
		=&&d(vec\{\E[\boldsymbol X_{i}|\mathcal F_{-\infty}^{i}],\E[\boldsymbol X_{i+1}|\mathcal F_{-\infty}^{i+1}],...,\E[\boldsymbol X_{i+(l-1)}|\mathcal F_{-\infty}^{i+(l-1)}]\})\\
		&&-d(vec\{\E[\boldsymbol X_{i}|\mathcal F_{i-sl}^{i}],\E[\boldsymbol X_{i+1}|\mathcal F_{i+1-sl}^{i+1}],...,\E[\boldsymbol X_{i+(l-1)}|\mathcal F_{i+(l-1)-sl}^{i+(l-1)}]\})\\
		=&&\sum_{j=sl+1-i}^{\infty}\big(d(vec\{\E[\boldsymbol X_{i}|\mathcal F_{-j}^{i}],\E[\boldsymbol X_{i+1}|\mathcal F_{-j+1}^{i+1}],...,\E[\boldsymbol X_{i+(l-1)}|\mathcal F_{-j+(l-1)}^{i+(l-1)}]\})\\
		&&\quad -d(vec\{\E[\boldsymbol X_{i}|\mathcal F_{-j+1}^{i}],\E[\boldsymbol X_{i+1}|\mathcal F_{-j+2}^{i+1}],...,\E[\boldsymbol X_{il}|\mathcal F_{-j+l}^{i+(l-1)}]\})\big)\\
		=&&\sum_{j=sl+1-i}^{\infty}\big(d(vec\{\E[\boldsymbol G_{-j}^{i}|\mathcal F_{-j}^{i}],\E[\boldsymbol G_{-j+1}^{i+1}|\mathcal F_{-j+1}^{i+1}],...,\E[\boldsymbol G_{-j+(l-1)}^{i+(l-1)}|\mathcal F_{-j+(l-1)}^{i+(l-1)}]\})\\
		&&\quad -d(vec\{\E[(\boldsymbol G')_{-j}^{i}|\mathcal F_{-j}^{i}],\E[(\boldsymbol G')_{-j+1}^{i+1}|\mathcal F_{-j+1}^{i+1}],...,\E[(\boldsymbol G')_{-j+(l-1)}^{i+(l-1)}|\mathcal F_{-j+(l-1)}^{i+(l-1)}]\})\big).
	\end{eqnarray*}
	
	Since $\mathcal D$ is set of Lipschitz function with Lipschitz constant $K_{\mathcal{D}}$, combining equality above with condition \eqref{e:GMC}, we can get 
	\begin{eqnarray}\label{e:3.4}
		&& \E\sup_{d\in\mathcal D}|\sum_{i=1}^{Q_k}(d(\boldsymbol X_{k+l(i-1)}^\prime)-d(\tilde{\boldsymbol X}_{k+l(i-1)}^\prime))|\\
		&\le& \E\sup_{d\in\mathcal D}\sum_{i=1}^{Q_k}\sum_{j=(s-i+1)l+1-k}^{\infty}|d(vec\{\E[\boldsymbol G_{-j}^{k+l(i-1)}|\mathcal F_{-j}^{k+l(i-1)}],...,\E[\boldsymbol G_{-j+l-1}^{k+li-1}|\mathcal F_{-j+l-1}^{k+li-1}]\})\nonumber\\
		&&\quad -d(vec\{\E[(\boldsymbol G')_{-j}^{k+l(i-1)}|\mathcal F_{-j}^{k+l(i-1)}],...,\E[(\boldsymbol G')_{-j+l-1}^{k+li-1}|\mathcal F_{-j+l-1}^{k+li-1}]\})|\nonumber\\
		&\le& K_{\mathcal{D}}\E\big\{\sum_{i=1}^{Q_k}\sum_{j=(s-i+1)l+1-k}^{\infty}\|vec\{\E[\boldsymbol G_{-j}^{k+l(i-1)}-(\boldsymbol G')_{-j}^{k+l(i-1)}\big|\mathcal F_{-j}^{k+l(i-1)}],...,\nonumber\\
		&&\quad \E[\boldsymbol G_{-j+l-1}^{k+li-1}-(\boldsymbol G')_{-j+l-1}^{k+li-1}\big|\mathcal F_{-j+l-1}^{k+li-1}] \}\| \big\}\nonumber\\
		&\le& K_{\mathcal{D}}\sum_{i=1}^{Q_k}\sum_{j=(s-i+1)l+1-k}^{\infty}\{\E\|\boldsymbol G_{-j}^{k+l(i-1)}-(\boldsymbol G')_{-j}^{k+l(i-1)}\|+...+\E\|\boldsymbol G_{-j+l-1}^{k+li-1}-(\boldsymbol G')_{-j+l-1}^{k+li-1}\|\}\nonumber\\
		&\le& lK_{\mathcal{D}}\sum_{i=1}^{Q_k}\sum_{j=(s-i+1)l+1-k}^{\infty} C_1^{1/\beta}e^{-C_2(k+l(i-1)+j)/\beta}
		= Q_kC_1^{1/\beta}lK_{\mathcal{D}}e^{-C_2(sl+1)/\beta}/(1-e^{-C_2/\beta})\nonumber\\
		&\le&  nC_1^{1/\beta}K_{\mathcal{D}}e^{-C_2(sl+1)/\beta}/(1-e^{-C_2/\beta})\nonumber.
	\end{eqnarray}
	
	Since 
	$$
	\sup _{d \in \mathcal{D}}|\sum_{i=1}^{Q_k}(\mathbb{E}_{\pi^l}\left[d(\boldsymbol X_{k+l(i-1)}^\prime)\right]-\mathbb{E}_{\tilde{{\pi^l}}}[d(\tilde{\boldsymbol X}_{k+l(i-1)}^\prime)])|\le \E\sup_{d\in\mathcal D}|\sum_{i=1}^{Q_k}(d(\boldsymbol X_{k+l(i-1)}^\prime)-d(\tilde{\boldsymbol X}_{k+l(i-1)}^\prime))|,
	$$ 
	then the result in \eqref{e:3.4} yields that
	$$
	\sup _{d \in \mathcal{D}}|\sum_{i=1}^{Q_k}(\mathbb{E}_{\pi^l}[d(\boldsymbol X_{k+l(i-1)}^\prime)]-\mathbb{E}_{\tilde{{\pi^l}}}[d(\tilde{\boldsymbol X}_{k+l(i-1)}^\prime)])|\le nC_1^{1/\beta}K_{\mathcal{D}}e^{-C_2(sl+1)/\beta}/(1-e^{-C_2/\beta}).
	$$
	
	(ii) The Chebyshev inequality and the result in (i) imply
	\begin{eqnarray*}
		\PP\Big(\sup_{d\in\mathcal D}\big\{\sum_{i=1}^{Q_k}d(\boldsymbol X_{k+l(i-1)}^\prime)-\sum_{i=1}^{Q_k}d(\tilde{\boldsymbol X}_{k+l(i-1)}^\prime)  \big\}\ge x\Big)&\le& \frac1x \E\sup_{d\in\mathcal D}\big\{\sum_{i=1}^{Q_k}(d(\boldsymbol X_{k+l(i-1)}^\prime)-d(\tilde {\boldsymbol X}_{k+l(i-1)}^\prime))  \big\}\\
		&\le& 	\frac1{x(1-e^{-C_2/\beta})} nC_1^{1/\beta}K_{\mathcal{D}}e^{-C_2(sl+1)/\beta}.
	\end{eqnarray*}	
	
	We can obtain with probability at least $1-e^{-C_2sl/(2\beta)}$,
	\begin{eqnarray*}
		\sup_{d\in\mathcal D}\big\{\sum_{i=1}^{Q_k}d(\boldsymbol X_{k+l(i-1)}^\prime)-\sum_{i=1}^{Q_k}d(\tilde{\boldsymbol X} _{k+l(i-1)}^\prime)  \big\}
		\le nC_1^{1/\beta}K_{\mathcal{D}}e^{-C_2sl/(2\beta)-C_2/\beta}/({1-e^{-C_2/\beta}}).
	\end{eqnarray*}		
\end{proof}

Given a real-valued function class $\mathcal{H}$ and some set of data points {  $\boldsymbol x_1,...,\boldsymbol x_n\in\R^d$}, we define the (empirical) Rademacher complexity $\hat{\mathcal{R}}_n(\mathcal{H})$ as
$$
\hat{\mathcal{R}}_n(\mathcal{H}):=\mathbb{E}_{\boldsymbol\eta}\left[\sup _{h \in \mathcal{H}} \frac{1}{n} \sum_{i=1}^n \eta_i h\left(\boldsymbol x_i\right)\right],
$$
where $\boldsymbol\eta=\left(\eta_1, \ldots, \eta_n\right)$ is a vector uniformly distributed in $\{-1,+1\}^n$.

Theorem 2 in \citet{golowich2018size} shows that, 
\begin{lemma}[Rademacher Complexity of Neural Networks]\label{lem:complexity}
	Let $\mathcal{H}_L$ be the class of real-valued networks of depth L, where $\left\|W_j\right\|_{1, \infty} \leq$ $M(j)$ for all $j \in\{1, \ldots, L\}$, and $\sigma$ is a 1-Lipschitz activation function with $\sigma(0)=0$, applied element-wise. Then
	\begin{eqnarray*}
		\hat{\mathcal{R}}_n\left(\mathcal{H}_L\right) \leq \frac{2}{n} \prod_{j=1}^L M(j) \cdot \sqrt{L+1+\log (p)} \cdot \sqrt{\max _{j \in\{1, \ldots, p\}} \sum_{i=1}^n \boldsymbol x_{i, j}^2},
	\end{eqnarray*}
	where $\boldsymbol x_{i, j}$ is the j-th coordinate of the vector $\boldsymbol{x}_i$.
\end{lemma}
\begin{proof}[Proof of Lemma \ref{lem:concentration}] 
	For $j \in$ $\{1,2, \ldots, \lceil Q_k /(2 s)\rceil\}$, let
	\begin{eqnarray*}
		I_j^k&=&\{i=k+2sl(j-1),k+l+2sl(j-1),...,k+(s-1)l+2sl(j-1)\},\\
		J_j^k&=&\{i=k+sl+2sl(j-1),k+l(s+1)+2sl(j-1),...,k+(2s-1)l+2sl(j-1)\},
	\end{eqnarray*}
	{thus we can decompose $[2sl\lceil Q_k /(2 s)\rceil]$ as a union of $2\lceil Q_k /(2 s)\rceil$ blocks.} Let 
	\begin{eqnarray*}
		\boldsymbol{Y}^k_{I, j}=\sum_{i \in I_j^k} d\left(\tilde{\boldsymbol X_{i}^\prime}\right), 
		\boldsymbol{Y}^k_{J, j}=\sum_{i \in J_j^k} d\left(\tilde{\boldsymbol X_{i}^\prime}\right) .
	\end{eqnarray*}
	By the definition of $\tilde{\boldsymbol X}_i^\prime$, we know $\{\boldsymbol{Y}^k_{I, j}\}_{1\leq j\leq \lceil Q_k /(2 s)\rceil}$ is an i.i.d. sequence of random variables, and $\{\boldsymbol{Y}^k_{J, j}\}_{1\leq j\leq \lceil Q_k /(2 s)\rceil}$ is the same.
	
	By Lemma \ref{lem:difdepen}, it is easy to see that, with probability at least $1-le^{-C_2sl/(2\beta)}$
	\begin{eqnarray*}
		&&\sup _{d \in \mathcal{D}}\left|\frac{1}{\bar n} \sum_{i=1}^{\bar  n}\left(d\left(\boldsymbol X_{i}^\prime\right)-\mathbb{E}_{{\pi^l}}\left[d\left(\boldsymbol X_{i}^\prime\right)\right]\right)\right|\\
		&=&\sup _{d \in \mathcal{D}}\frac{1}{\bar n}\left| \sum_{k=1}^{l}\sum_{i=1}^{Q_k}\left(d\left(\boldsymbol X_{k+l(i-1)}^\prime\right)-\mathbb{E}_{{\pi^l}}\left[d\left(\boldsymbol X_{k+l(i-1)}^\prime\right)\right]\right)\right|\\
		&\le&\sup _{d \in \mathcal{D}}\frac{1}{\bar n}\sum_{k=1}^{l}\left| \sum_{i=1}^{Q_k}\left(d\left(\boldsymbol X_{k+l(i-1)}^\prime\right)-\mathbb{E}_{ {\pi^l}}\left[d\left({\boldsymbol X}_{k+l(i-1)}^\prime\right)\right]\right)- \sum_{i=1}^{Q_k}\left(d\left(\tilde{\boldsymbol {X}}_{k+l(i-1)}^\prime\right)-\mathbb{E}_{\tilde{{\pi^l}}}\left[d\left(\tilde {\boldsymbol X}_{k+l(i-1)}^\prime\right)\right]\right)\right|\\
		&&+\sup _{d \in \mathcal{D}}\frac{1}{\bar n}\sum_{k=1}^{l}\left|\sum_{i=1}^{Q_k}\left(d\left(\tilde {\boldsymbol X}_{k+l(i-1)}^\prime\right)-\mathbb{E}_{\tilde{{\pi^l}}}\left[d\left(\tilde {\boldsymbol X}_{k+l(i-1)}^\prime\right)\right]\right)\right|\\
		&\le&\frac{1}{\bar n}\left( nlC_1^{1/\beta}K_{\mathcal{D}}e^{-C_2(sl+1)/\beta}/(1-e^{-C_2/\beta})+nlC_1^{1/\beta}K_{\mathcal{D}}e^{-C_2sl/(2\beta)-C_2/\beta}/({1-e^{-C_2/\beta}})\right)\\
		&&+\sup _{d \in \mathcal{D}}\frac{1}{\bar n}\sum_{k=1}^{l}\left| \sum_{j=1}^{\lceil Q_k/2s\rceil }\left(\boldsymbol Y^k_{I,j}-\mathbb{E}_{\tilde {\pi^l}}\left[\boldsymbol Y^k_{I,j}\right]\right)\right|+\sup _{d \in \mathcal{D}}\frac{1}{\bar n}\sum_{k=1}^{l}\left| \sum_{j=1}^{\lceil Q_k/2s\rceil }\left(\boldsymbol Y^k_{J,j}-\mathbb{E}_{\tilde {\pi^l}}\left[\boldsymbol Y^k_{J,j}\right]\right)\right|.
	\end{eqnarray*}
	
	
	{  Denote $Z^k_I:=\sup _{d \in \mathcal{D}}|\frac{1}{\lceil {Q_k}/{2s}\rceil } \sum_{j=1}^{\lceil Q_k/2s\rceil }(\boldsymbol Y^k_{I,j}-\mathbb{E}_{\tilde {\pi^l}}[\boldsymbol Y^k_{I,j}])|$.} Since $\mathbb{E}\left(\boldsymbol Y^k_{I,j}\right)-\mathbb{E}_{\tilde {\pi^l}}\left[\boldsymbol Y^k_{I,j}\right]=0$ and for all $1\leq j\leq \lceil Q_k/2s\rceil $, 
	\begin{eqnarray*}
		|\boldsymbol Y^k_{I,j}-\mathbb{E}_{\tilde {\pi^l}}\left[\boldsymbol Y^k_{I,j}\right]|&\le&\sum_{i=1}^s\left|d\left(\tilde {\boldsymbol X}_{k+2sl(j-1)+(i-1)l}^\prime\right)-\mathbb{E}_{\tilde {\pi^l}}\left[d\left(\tilde {\boldsymbol X}_{k+2sl(j-1)+(i-1)l}^\prime\right)\right]\right|\\ &\le&\sum_{i=1}^s(\left|d\left(\tilde {\boldsymbol X}_{k+2sl(j-1)+(i-1)l}^\prime\right)\right|+\mathbb{E}_{\tilde {\pi^l}}\left|d\left(\tilde {\boldsymbol X}_{k+2sl(j-1)+(i-1)l}^\prime\right)\right|) \leq 2sB,
	\end{eqnarray*}
	by Massart's concentration inequality (cf. \citet[Theorem 9]{massart2000constants}) for suprema of the bounded empirical processes, we have for all $t_1>0$
	\begin{eqnarray*}
		\mathbb{P}\left\{Z_I^k \geq \mathbb{E} Z_I^k+4s B \sqrt{\frac{2t_1}{\lceil {Q_k}/{2s}\rceil}}\right\} \leq \mathrm{e}^{-t_1}.
	\end{eqnarray*}
	
	For the expectation, we consider the Rademacher random variables $\{\eta_i^j\}_{i=1, \ldots, I_j^k}, 1 \leq j \leq \lceil Q_k / 2 s\rceil $, that is $\eta_i^j$ equals to 1 and $-1$ with propability $1 / 2$. Then by using symmetrization theorem for empirical processes (see Lemma 2.3.1 in van der Vaart and Wellner,1996) and Lemma \ref{lem:complexity}, we can get
	\begin{eqnarray*}
		\mathbb{E} Z_I^k&=&\mathbb{E} \sup _{d \in \mathcal{D}}\left|\frac{1}{\lceil Q_k/2s\rceil } \sum_{j=1}^{\lceil Q_k/2s\rceil }\left(\boldsymbol Y^k_{I,j}-\mathbb{E}_{\tilde {\pi^l}}\left[\boldsymbol Y^k_{I,j}\right]\right)\right|\\
		& =& \mathbb{E} \sup _{d \in \mathcal{D}}\left|\frac{1}{\lceil Q_k/2s\rceil } \sum_{j=1}^{\lceil Q_k/2s\rceil}(\sum_{i \in I^k_j} d(\tilde{\boldsymbol{X}}_i^\prime)-\mathbb{E}_{\tilde {\pi^l}}[\sum_{i \in I^k_j} d(\tilde{\boldsymbol{X}}_i^\prime)])\right|\\
		& =&\mathbb{E} \sup _{d \in \mathcal{D}}\left| \frac{1}{\lceil Q_k/2s\rceil }\sum_{j=1}^{\lceil Q_k/2s\rceil} \sum_{i \in I^k_j} d\left(\tilde{\boldsymbol{X}}_i^\prime\right)-\mathbb{E}_{\tilde{{\pi^l}}}[d]\right|\\
		& \le& \frac{2}{\lceil Q_k/2s\rceil} \mathbb{E} \sup _{d \in \mathcal{D}}\left|\sum_{j=1}^{\lceil Q_k/2s\rceil} \sum_{i \in I^k_j} \eta_i^j d\left(\tilde{\boldsymbol{X}}_i^\prime\right)\right|\\
		&\le&\frac{4}{\lceil Q_k/2s\rceil} K_{\mathcal{D}} \cdot \sqrt{L_{\mathcal{D}}+1+\log (lp)} \cdot \mathbb{E} \sqrt{\max _{w \in\{1, \ldots, lp\}} \sum_{j=1}^{\lceil Q_k/2s\rceil} \sum_{i \in I^k_j} (\tilde{\boldsymbol{X}}^\prime_{i,w})^2},
	\end{eqnarray*}
	where $\tilde{\boldsymbol{X}}^\prime_{i,w}$ denote the $w$ th coordinate of $\tilde{\boldsymbol X}_i^\prime$.
	
	Assumption \ref{assum2} or \ref{assum4} implies that stationary distribution $\pi$ of time series $\{\boldsymbol X_n\}_{n\in\N}$ has second order moment bound and thus
	\begin{eqnarray*}
		\mathbb{E} \max _{w \in\{1, \ldots, lp\}} (\tilde{\boldsymbol{X}}^\prime_{i,w})^2 &\le& \mathbb{E} \sum_{w=1}^{lp} (\tilde{\boldsymbol{X}}^\prime_{i,w})^2\\
		&=&\sum_{w=1}^{lp} \mathbb{E}\left[vec\{\E[\boldsymbol X_{i}|\mathcal F_{i-sl}^{i}],...,\E[\boldsymbol X_{i+l-1}|\mathcal F_{i+l-1-sl}^{i+l-1}]\}_w^2\right]\\
		&\le& \sum_{w=1}^{p} (\mathbb{E} \boldsymbol{X}_{{i}, w}^2+\mathbb{E} \boldsymbol{X}_{{i+1}, w}^2+...+\mathbb{E} \boldsymbol{X}_{i+l-1, w}^2 )\leq C_3l.
	\end{eqnarray*}
	Jensen's inequality implies
	\begin{eqnarray*}
		\mathbb{E} \sqrt{\max _{w \in\{1, \ldots, lp\}} \sum_{j=1}^{\lceil Q_k/2s\rceil} \sum_{i \in I^k_j} (\tilde{\boldsymbol{X}}^\prime_{i,w})^2}  &\leq&\left\{\mathbb{E} \max_{w \in\{1, \ldots, lp\}} \sum_{j=1}^{\lceil Q_k/2s\rceil} \sum_{i \in I^k_j} (\tilde{\boldsymbol{X}}_i^\prime)_{ w}^2\right\}^{1 / 2}\\
		&\leq&\left\{\sum_{j=1}^{\lceil Q_k/2s\rceil} \sum_{i \in I^k_j} \mathbb{E} \max _{w \in\{1, \ldots, lp\}} (\tilde{\boldsymbol{X}}^\prime_{i,w})^2\right\}^{1 / 2}\leq  \sqrt{C_3l(s+Q_k/2)}.
	\end{eqnarray*}
	Hence, it is easy to know that, 
	\begin{eqnarray*}
		\mathbb{P}\left\{Z^k_I \geq 4\frac{\sqrt{C_3l(s+Q_k/2)}}{\lceil {Q_k}/{2s}\rceil}  K_{\mathcal{D}} \cdot \sqrt{L_{\mathcal{D}}+1+\log (lp)}+4s B \sqrt{\frac{2t_1}{\lceil {Q_k}/{2s}\rceil}}\right\} \leq \mathrm{e}^{-t_1}.
	\end{eqnarray*}
	{  Let $Z_J^k:=\left\{\boldsymbol{Y}^k_{J, j}\right\}_{1 \leq j \leq \lceil {Q_k}/{2s}\rceil}$, one can get similarly 
		\begin{eqnarray*}
			\mathbb{P}\left\{Z_J^k \geq 4\frac{\sqrt{C_3l(s+Q_k/2)}}{\lceil {Q_k}/{2s}\rceil}K_{\mathcal{D}} \cdot \sqrt{L_{\mathcal{D}}+1+\log (lp)}+4s B \sqrt{\frac{2t_1}{\lceil {Q_k}/{2s}\rceil}}\right\} \leq \mathrm{e}^{-t_1}.
		\end{eqnarray*}
	}
	
	Since $Q_k\leq n/l$, it is easy to know that, with probability at least $1-2l\mathrm{e}^{-t_1}-le^{-C_2sl/(2\beta)}$,
	\begin{eqnarray*}
		&&\sup _{d \in \mathcal{D}}\left|\frac{1}{\bar n} \sum_{i=1}^{\bar n}\left(d\left({\boldsymbol X}_i^\prime\right)-\mathbb{E}_{{\pi^l}}\left[d\left({\boldsymbol X}_i^\prime\right)\right]\right)\right|\\
		&\le&\frac{1}{\bar n}\left( nlC_1^{1/\beta}K_{\mathcal{D}}e^{-C_2(sl+1)/\beta}/(1-e^{-C_2/\beta})+nlC_1^{1/\beta}K_{\mathcal{D}}e^{-C_2sl/(2\beta)-C_2/\beta}/({1-e^{-C_2/\beta}})\right)\\
		&&+\frac{1}{\bar n}\sum_{k=1}^{l}\left[4{\sqrt{C_3l(s+Q_k/2)}}K_{\mathcal{D}} \cdot \sqrt{L_{\mathcal{D}}+1+\log (lp)}+4s B \sqrt{{2t_1\lceil {Q_k}/{2s}\rceil}}\right]\\
		&\le&\frac{1}{\bar n}\left( nlC_1^{1/\beta}K_{\mathcal{D}}e^{-C_2(sl+1)/\beta}/(1-e^{-C_2/\beta})+nlC_1^{1/\beta}K_{\mathcal{D}}e^{-C_2sl/(2\beta)-C_2/\beta}/({1-e^{-C_2/\beta}})\right)\\
		&&+\frac{1}{\bar n}C_3^\prime\left[l\sqrt{2sl+n}  K_{\mathcal{D}} \cdot \sqrt{L_{\mathcal{D}}+1+\log (lp)}+B \sqrt{t_1snl}\right].
	\end{eqnarray*}
\end{proof}
\subsubsection{Generator approximation error}
\begin{lemma}\label{prop:3.6}
	Suppose that $\mathcal{G}=\mathcal{N N}\left(N_{\mathcal{G}}, L_{\mathcal{G}}\right)$ satisfies $N_{\mathcal{G}} \geq 7 lp+1, L_{\mathcal{G}} \geq 2$ and
	$$
	\bar n \leq \frac{N_{\mathcal{G}}-lp-1}{2}\left\lfloor\frac{N_{\mathcal{G}}-lp-1}{6l p}\right\rfloor\left\lfloor\frac{L_{\mathcal{G}}}{2}\right\rfloor+2 .
	$$
	For any $\mathcal{D}\subset \{d\in\mathcal{N} \mathcal{N}\left(N_{\mathcal{D}}, L_{\mathcal{D}}, K_{\mathcal{D}}\right): d(\mathbb{R}^{lp}) \subseteq[-B,B]\}\cap \operatorname{Lip}\left(\mathbb{R}^{lp}, C\right) ,B,C<\infty
	$, we have
	\begin{eqnarray*}
		\inf_{\boldsymbol g\in\mathcal G}\sup_{d\in\mathcal D}\Big\{ \frac1{\bar n}\sum_{i=1}^{\bar n}d({\boldsymbol X}_i^\prime)-\E_{\boldsymbol g_\#\nu}[d]\Big\}=0.
	\end{eqnarray*}
\end{lemma}

\begin{proof}
	Suppose that $\mathcal{G}=\mathcal{N N}\left(N_{\mathcal{G}}, L_{\mathcal{G}}\right)$ satisfies $N_{\mathcal{G}} \geq 7 lp+1, L_{\mathcal{G}} \geq 2$ and
	$$
	\bar n \leq \frac{N_{\mathcal{G}}-lp-1}{2}\left\lfloor\frac{N_{\mathcal{G}}-lp-1}{6l p}\right\rfloor\left\lfloor\frac{L_{\mathcal{G}}}{2}\right\rfloor+2 .
	$$
	
	For any $\mathcal{D}\subset \{d\in\mathcal{N} \mathcal{N}\left(N_{\mathcal{D}}, L_{\mathcal{D}}, K_{\mathcal{D}}\right): d(\mathbb{R}^{lp}) \subseteq[-B,B]\}\cap \operatorname{Lip}\left(\mathbb{R}^{lp}, C\right) ,B,C<\infty
	$, 
	we have
	\begin{eqnarray*}
		\inf_{\boldsymbol g\in\mathcal G}\sup_{d\in\mathcal D}\Big\{ \frac1{\bar n}\sum_{i=1}^{\bar n}d({\boldsymbol X}_i^\prime)-\E_{\boldsymbol g_\#\nu}[d]\Big\} 
		&\le&\inf _{\boldsymbol g \in \mathcal{G}} \sup _{f \in \text{Lip}(\mathbb{R}^{lp},C)} \int f\left(\dif \hat{{\pi^l}}_{\bar n}-\dif \boldsymbol g_{\#}\nu \right) \\
		&\le& \inf _{\boldsymbol g \in \mathcal{G}} C\sup _{\|f\|_{\text {Lip }} \leq 1} \int f\left(\dif \hat{{\pi^l}}_{\bar n}-\dif \boldsymbol g_{\#} \nu\right) \\
		&=&C\inf _{\boldsymbol g \in \mathcal{G}} \mathcal{W}_1\left(\hat{{\pi^l}}_{\bar n}, \boldsymbol g_{\#}\nu \right),
	\end{eqnarray*}
	where $\bar n=n-l+1$, $\hat{\pi^l}_{\bar n}:=\frac1{\bar n}\sum_{i=1}^{\bar n}\delta_{x_i}$ is the empirical distribution of $\{{\boldsymbol X}_i^\prime\}_{i=1,...,\bar n}$.
	
	Since $\mathcal{W}_1\left(\hat{{\pi^l}}_{\bar n}, \boldsymbol g_{\epsilon \#} \nu\right) \rightarrow 0$ for some $\boldsymbol g_\epsilon$ as $\epsilon \rightarrow 0$ by Lemma $3.2$ of \citet{yang2022capacity}. 
	$$
	\inf_{\boldsymbol g\in\mathcal G}\sup_{d\in\mathcal D}\Big\{ \frac1{\bar n}\sum_{i=1}^{\bar n}d({\boldsymbol X}_i^\prime)-\E_{\boldsymbol g_\#\nu}[d]\Big\}=0.
	$$
\end{proof}
\subsubsection{Statistical error for simulated sample}
\begin{lemma}\label{prop:3.7}
	Suppose $\mathcal{G}=\mathcal{N} \mathcal{N}\left(N_{\mathcal{G}}, L_{\mathcal{G}},K_{\mathcal{G}}\right)$, and
	$$
	\mathcal{D}=\{d\in\mathcal{N} \mathcal{N}\left(N_{\mathcal{D}}, L_{\mathcal{D}}, K_{\mathcal{D}}\right): d(\mathbb{R}^{lp}) \subseteq[-B,B]\} ,B<\infty,
	$$
	then with probability at least $1-e^{-t_2}$,
	\begin{eqnarray*}
		d_{\mathcal{D} \circ \mathcal{G}}\left(\widehat{\nu}_m, \nu\right) \leq \frac{2}{m} K_{\mathcal{D}}K_{\mathcal{G}}\sqrt{L_{\mathcal{D}}+L_{\mathcal{G}}+2} \cdot \mathbb{E} \max _{i \in[m]}\left\|\boldsymbol z_i\right\|_2+2 B \sqrt{\frac{8 t_2}{m}}.
	\end{eqnarray*}
\end{lemma}
\begin{proof}
	Statistical error for simulated sample is defined as
	\begin{eqnarray*}
		d_{\mathcal{D} \circ \mathcal{G}}\left(\hat{\nu}_m, \nu\right)=\sup _{d \circ\boldsymbol g \in \mathcal{D} \circ \mathcal{G}}\left|\frac{1}{m} \sum_{i=1}^m\left((d \circ\boldsymbol g)\left(\boldsymbol{z}_i\right)-\mathbb{E}_{\boldsymbol{z}_i \sim \nu}\left[(d \circ \boldsymbol g)\left(\boldsymbol{z}_i\right)\right]\right)\right|.
	\end{eqnarray*}
	
	Let $\mathcal{R}_{m}(\mathcal{D}\circ \mathcal{G} ; \boldsymbol{z}):=\mathbb{E} \sup _{d\circ \boldsymbol g \in \mathcal{D}\circ \mathcal{G}}\left|\frac{1}{m} \sum_{i=1}^m \eta_i d\circ \boldsymbol g\left(\boldsymbol{z}_i\right)\right|$ is the Rademacher complexity of functional class $\mathcal{D}\circ \mathcal{G}$, and $\eta_1, \ldots, \eta_m$ are independent Rademacher random variables. 
	
	We can obtain $\mathbb{E} d_{\mathcal{D} \circ \mathcal{G}}\left(\widehat{\nu}_m, \nu\right) \leq 2 \mathcal{R}_m(\mathcal{D} \circ \mathcal{G} ; \boldsymbol{z})$ by the symmetrization theorem. Using assumption $\sup _{d \in \mathcal{D}}\|d\|_{\infty} \leq B$, Massart's concentration inequality for suprema of the bounded empirical processes shows that
	\begin{eqnarray*}
		&&\mathbb{P}\left\{d_{\mathcal{D} \circ \mathcal{G}}\left(\widehat{\nu}_m, \nu\right) \geq 2 \mathcal{R}_m(\mathcal{D} \circ \mathcal{G})+2 B \sqrt{\frac{8 t_2}{m}}\right\} \\
		&\leq& \mathbb{P}\left\{d_{\mathcal{D} \circ \mathcal{G}}\left(\widehat{\nu}_m, \nu\right) \geq \mathbb{E} d_{\mathcal{D} \circ \mathcal{G}}\left(\widehat{\nu}_m, \nu\right)+2 B \sqrt{\frac{8 t_2}{m}}\right\} \leq \mathrm{e}^{-t_2}.
	\end{eqnarray*}
	Next, we need to calculate $\mathcal{R}_m(\mathcal{D} \circ \mathcal{G})$. 
	
	Since the composition of two neural networks $d \in \mathcal{D}=\mathcal{N} \mathcal{N}\left(N_{\mathcal{D}}, L_{\mathcal{D}},K_{\mathcal{D}}\right)$ and $\boldsymbol g \in \mathcal{G}=\mathcal{N} \mathcal{N}\left(N_{\mathcal{G}}, L_{\mathcal{G}},K_{\mathcal{G}}\right)$ is still a neural network with width $\max \left\{N_{\mathcal{D}}, N_{\mathcal{G}}\right\}$ and depth $L_{\mathcal{D}}+L_{\mathcal{G}}+1$. 
	Lemma \ref{lem:complexity} implies
	\begin{eqnarray*}
		\mathbb{E} d_{\mathcal{D} \circ \mathcal{G}}\left(\widehat{\nu}_m, \nu\right) \leq 2 \mathcal{R}_m(\mathcal{D} \circ \mathcal{G} ;\boldsymbol z) \leq \frac{2}{m} K_{\mathcal{D}}K_{\mathcal{G}}\sqrt{L_{\mathcal{D}}+L_{\mathcal{G}}+2} \cdot \mathbb{E} \max _{i \in[m]}\left\|\boldsymbol z_i\right\|_2.
	\end{eqnarray*}
	
	Thus, with probability at least $1-e^{-t_2}$,
	\begin{eqnarray*}
		d_{\mathcal{D} \circ \mathcal{G}}\left(\widehat{\nu}_m, \nu\right) \leq \frac{2}{m} K_{\mathcal{D}}K_{\mathcal{G}}\sqrt{L_{\mathcal{D}}+L_{\mathcal{G}}+2} \cdot \mathbb{E} \max _{i \in[m]}\left\|\boldsymbol z_i\right\|_2+2 B \sqrt{\frac{8 t_2}{m}}.
	\end{eqnarray*}
\end{proof}

\subsection{Proof of Theorem \ref{thm:main}}
\begin{proof}[Proof of Theorem \ref{thm:main}]
	(i). By Lemmas \ref{lem:decom}(i), \ref{lem:concentration}, \ref{prop:3.6} and \ref{prop:3.7}, it is easy to see that, with probability at least $1-2l{e}^{-t_1}-le^{-C_2sl/2\beta}-e^{-t_2}$,
	\begin{eqnarray*}
		d_{\mathcal H}({\pi^l}, \boldsymbol {\hat g}_{\#}\nu)
		&\le&   2(2a_n)^b({  K_{\mathcal{D}}/\log^\gamma K_{\mathcal{D}}})^{-\frac b{lp+1}}+2(1+B)l p e^{-a_n+v^2 / 2}+
		\sup_{d\in\mathcal D}\Big\{\E_{\pi^l}[d]-\frac1{\bar n}\sum_{i=1}^{\bar n}d(\boldsymbol X_i^\prime)\Big\}\\
		&&+\inf_{\boldsymbol g\in\mathcal G}\sup_{d\in\mathcal D}\Big\{ \frac1{\bar n}\sum_{i=1}^{\bar n}d(\boldsymbol X_i^\prime)-\E_{\boldsymbol g_\#\nu}[d]\Big\}
		+2\sup_{d\circ\boldsymbol g\in\mathcal D\circ\mathcal G}\{\E_{\hat\nu_m}[d\circ\boldsymbol g]-\E_{\nu}[d\circ\boldsymbol g]\}\\
		&\leq&  2(2a_n)^b({  K_{\mathcal{D}}/\log^\gamma K_{\mathcal{D}}})^{-\frac b{lp+1}}+2(1+B) lp e^{-a_n+v^2 / 2}\\
		&&+\frac{1}{\bar n}\left( nlC_1^{1/\beta}K_{\mathcal{D}}e^{-C_2(sl+1)/\beta}/(1-e^{-C_2/\beta})+nlC_1^{1/\beta}K_{\mathcal{D}}e^{-C_2sl/(2\beta)-C_2/\beta}/({1-e^{-C_2/\beta}})\right)\\
		&&+\frac{1}{\bar n}C_3^\prime\left[l\sqrt{2sl+n}  K_{\mathcal{D}} \cdot \sqrt{L_{\mathcal{D}}+1+\log (lp)}+B \sqrt{t_1snl}\right]\\
		&&+\frac{4}{m} K_{\mathcal{D}}K_{\mathcal{G}}\sqrt{L_{\mathcal{D}}+L_{\mathcal{G}}+2} \cdot \mathbb{E} \max _{i \in[m]}\left\|\boldsymbol z_i\right\|_2+4 B \sqrt{\frac{8 t_2}{m}}.
		\end{eqnarray*}
	Since $a_n=\log (lpn)$ and when we choose $s=n^{\alpha}$, $0<\alpha<1$, it is easy to know that with probability at least $1-2l{e}^{-t_1}-le^{-C_2n^{\alpha}l/2\beta}-e^{-t_2}$,
		\begin{eqnarray*}
		d_{\mathcal H}({\pi^l}, \boldsymbol {\hat g}_{\#}\nu)&\le&   2(\log lpn)^b({  K_{\mathcal{D}}/\log^\gamma K_{\mathcal{D}}})^{-\frac b{lp+1}}+2(1+B) lp e^{-\log lpn+v^2 / 2}+\frac{n}{\bar n}\frac{C_1^{1/\beta}lK_{\mathcal{D}}}{e^{C_2(n^{\alpha}l+1)/\beta}(1-e^{-C_2/\beta})}\\
		&&+\frac{n}{\bar n}\frac{C_1^{1/\beta}l K_{\mathcal{D}}}{(1-e^{-C_2/\beta})e^{C_2n^{\alpha}l/2\beta+C_2/\beta}}+C_3^\prime l\frac{\sqrt{2n^\alpha l+n}}{\bar n}  K_{\mathcal{D}} \cdot \sqrt{L_{\mathcal{D}}+1+\log (lp)}\\
		&&+C_3^\prime B \sqrt{\frac{t_1ln^{\alpha+1}}{\bar n^2}}+\frac{4}{m} K_{\mathcal{D}}K_{\mathcal{G}}\sqrt{L_{\mathcal{D}}+L_{\mathcal{G}}+2} \cdot \mathbb{E} \max _{i \in[m]}\left\|\boldsymbol z_i\right\|_2+4 B \sqrt{\frac{8 t_2}{m}}\\
		&=& O\left((\log lp n)^b{  (K_{\mathcal{D}}/\log^\gamma K_{\mathcal{D}})}^{-\frac{b}{lp+1}}\right)+O\left({B}/{n}\right)+O\left(lK_{\mathcal{D}} e^{-C_2 n^{\alpha}l}\right)\\
		&&+O\left(l\sqrt{{1}/{n}} K_{\mathcal{D}} \sqrt{L_{\mathcal{D}}+\log lp}\right)+O\left(B\sqrt{{lt_1}/{n^{1-\alpha}}}\right)\\
		&&+O\left(K_{\mathcal{D}} K_{\mathcal{G}} \sqrt{L_{\mathcal{D}}+L_{\mathcal{G}}}/m+\sqrt{t_2/m}\right)\\
		&=& O\left((\log lpn)^b({  K_{\mathcal{D}}/\log^\gamma K_{\mathcal{D}}})^{-\frac{b}{lp+1}}\right)+O\left(lK_{\mathcal{D}} e^{-C_2 n^{\alpha}l}\right)+O\left(B\sqrt{{lt_1}/{n^{1-\alpha}}}\right)\\
		&&+O\left(l\sqrt{{1}/{n}} K_{\mathcal{D}} \sqrt{L_{\mathcal{D}}+\log lp}\right)+O\left(K_{\mathcal{D}} K_{\mathcal{G}} \sqrt{L_{\mathcal{D}}+L_{\mathcal{G}}}/{m} +\sqrt{{t_2}/{m}}\right).
	\end{eqnarray*}

{     
	(ii). By Lemmas \ref{lem:decom}(ii), \ref{lem:concentration}, \ref{prop:3.6},  \ref{prop:3.7} and setting $s=n^\alpha$ with $0<\alpha<1$, $a_n=n^{\frac{1-\alpha}{2\omega}}$, similarly as the calculation above, one has with probability at least
 $1-2l{e}^{-t_1}-le^{-C_2sl/2\beta}-e^{-t_2}$,
	\begin{eqnarray*}
	d_{\mathcal H}({\pi^l}, \boldsymbol {\hat g}_{\#}\nu)   
	&\le& O\left(n^{\frac{(1-\alpha)}{2\omega}}({  K_{\mathcal{D}}/\log^\gamma K_{\mathcal{D}}})^{-\frac{b}{lp+1}}\right)+O\left(lK_{\mathcal{D}} e^{-C_2 n^{\alpha}l}\right)+O\left(Bl\sqrt{{t_1}/{n^{1-\alpha}}}\right)\\
	&&+O\left(l\sqrt{{1}/{n}} K_{\mathcal{D}} \sqrt{L_{\mathcal{D}}+\log lp}\right)+O\left(K_{\mathcal{D}} K_{\mathcal{G}} \sqrt{L_{\mathcal{D}}+L_{\mathcal{G}}}/{m} +\sqrt{{t_2}/{m}}\right).
\end{eqnarray*}
}
\end{proof}

\noindent{\bf Acknowledgements:}  Lihu Xu is supported by National Natural Science Foundation of China No. 12071499 and University of Macau grant MYRG2020-00039-FST.
	
\bibliographystyle{plainnat}

\end{document}